\documentclass[runningheads]{llncs}
\usepackage{eccv}
\usepackage{eccvabbrv}

\usepackage[T1]{fontenc}
\usepackage{graphicx}

\usepackage{amsmath,amssymb,mathtools}
\usepackage{booktabs}
\usepackage{multirow}
\usepackage{microtype}

\usepackage[hidelinks]{hyperref}
\hypersetup{
  pdftitle={When 3D Gaussian Splatting Recovers Real Surfaces},
  pdfauthor={Songhe Wang and David Miller},
  pdfkeywords={3D Gaussian Splatting, geometric identifiability, spherical harmonics, light fields}
}
\AtBeginDocument{\crefname{assumption}{assumption}{assumptions}\Crefname{assumption}{Assumption}{Assumptions}}

\spnewtheorem{assumption}{Assumption}{\bfseries}{\itshape}

\newcommand{\Sph}{\mathbb S^2}
\newcommand{\R}{\mathbb R}
\newcommand{\E}{\mathbb E}
\newcommand{\PP}{\mathbb P}

\newcommand{\cJ}{\mathcal J}
\newcommand{\cV}{\mathcal V}

\newcommand{\norm}[1]{\left\lVert #1\right\rVert}

\DeclareMathOperator*{\dist}{dist}

\DeclareMathOperator{\Var}{Var}

\begin{document}
% ============================================================

\title{When 3D Gaussian Splatting Recovers Real Surfaces}
\titlerunning{When 3D Gaussian Splatting Recovers Real Surfaces}

\author{
Songhe Wang\inst{1} \and David Miller\inst{1}
}
\authorrunning{S. Wang and D. Miller}
\institute{
School of EECS, Penn State, University Park, PA, USA\\
\email{\{sxw5765,djm25\}@psu.edu}
}

\maketitle

\begin{figure*}[t]
  \centering
  \newlength{\colw}
  \setlength{\colw}{.20\textwidth} % three equal columns (slightly smaller)

  % ---- column headers (once) ----
  \begin{minipage}[t]{\colw}\centering \textbf{Ground Truth}\end{minipage}\hfill
  \begin{minipage}[t]{\colw}\centering \textbf{SH=3}\\Geometry Reconstruction\end{minipage}\hfill
  \begin{minipage}[t]{\colw}\centering \textbf{SH=24}\\Billboard Failure\end{minipage}

  % ================= (a) Appearance =================
  \begin{subfigure}{\textwidth}
    \centering
    \begin{minipage}[t]{\colw}\centering
      \includegraphics[width=\linewidth]{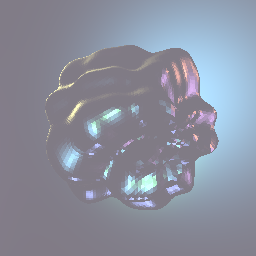}\\[.25ex]
      (Reference)
    \end{minipage}\hfill
    \begin{minipage}[t]{\colw}\centering
      \includegraphics[width=\linewidth]{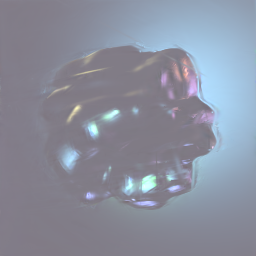}\\[.25ex]
      PSNR: 34.9\,dB (\(\uparrow\))
    \end{minipage}\hfill
    \begin{minipage}[t]{\colw}\centering
      \includegraphics[width=\linewidth]{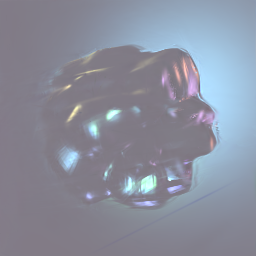}\\[.25ex]
      PSNR: 35.2\,dB (\(\uparrow\))
    \end{minipage}
    \caption{Rendering}
  \end{subfigure}

  % ================= (b) Geometry =================
  \begin{subfigure}{\textwidth}
    \centering
    \begin{minipage}[t]{\colw}\centering
      \includegraphics[width=\linewidth]{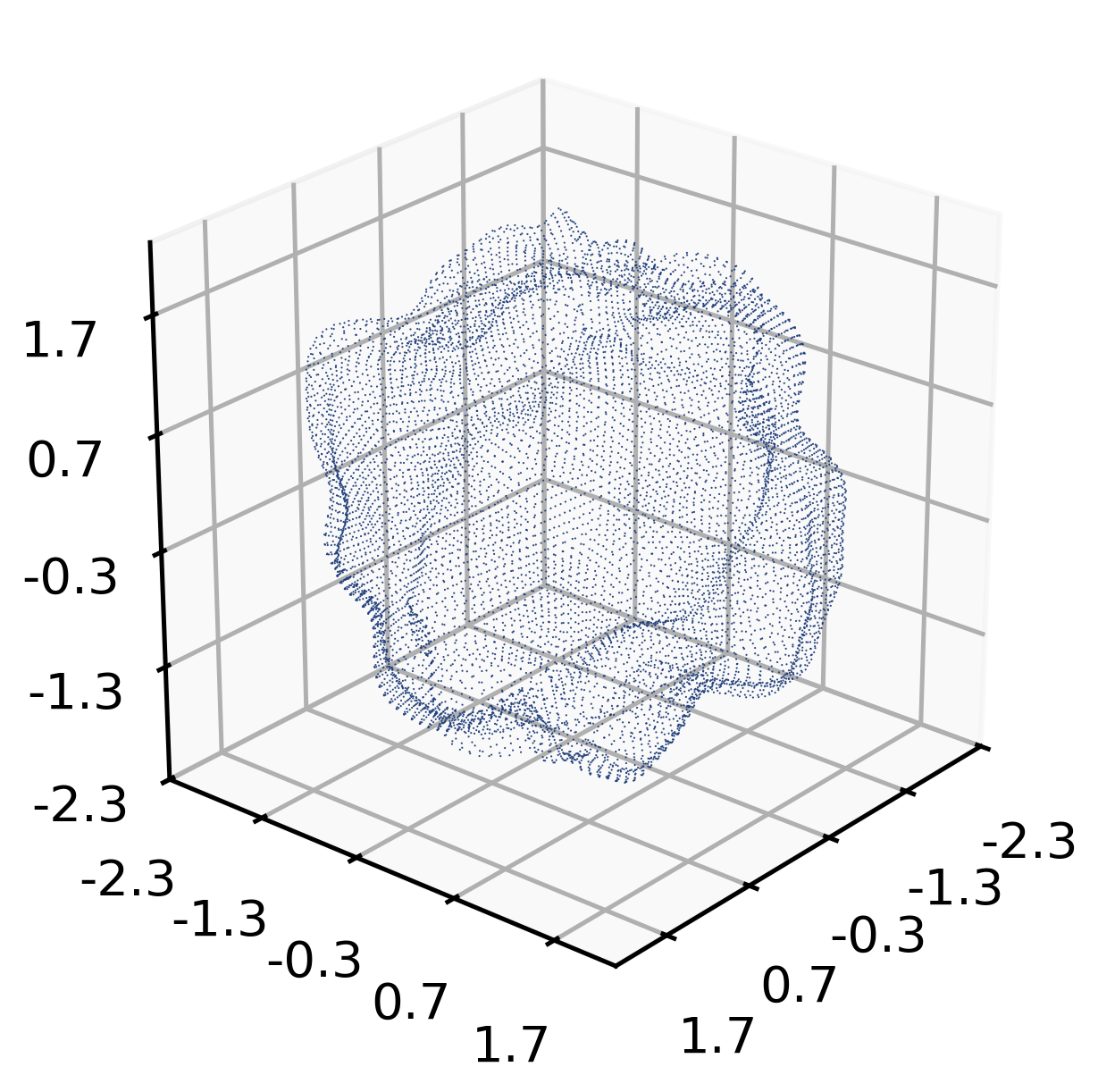}\\[.25ex]
      (Reference)
    \end{minipage}\hfill
    \begin{minipage}[t]{\colw}\centering
      \includegraphics[width=\linewidth]{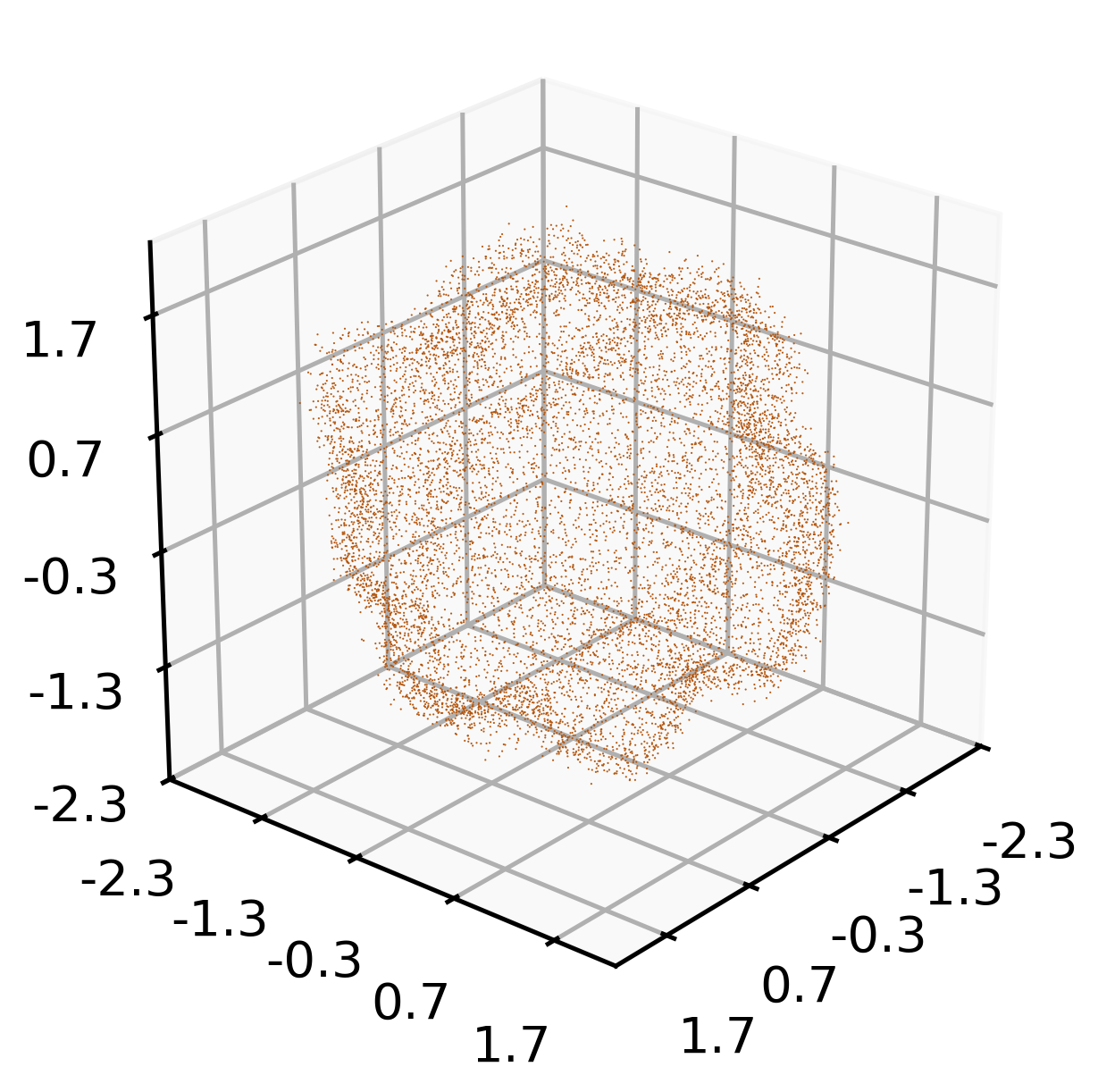}\\[.25ex]
      nCD: 0.03 (\(\downarrow\))
    \end{minipage}\hfill
    \begin{minipage}[t]{\colw}\centering
      \includegraphics[width=\linewidth]{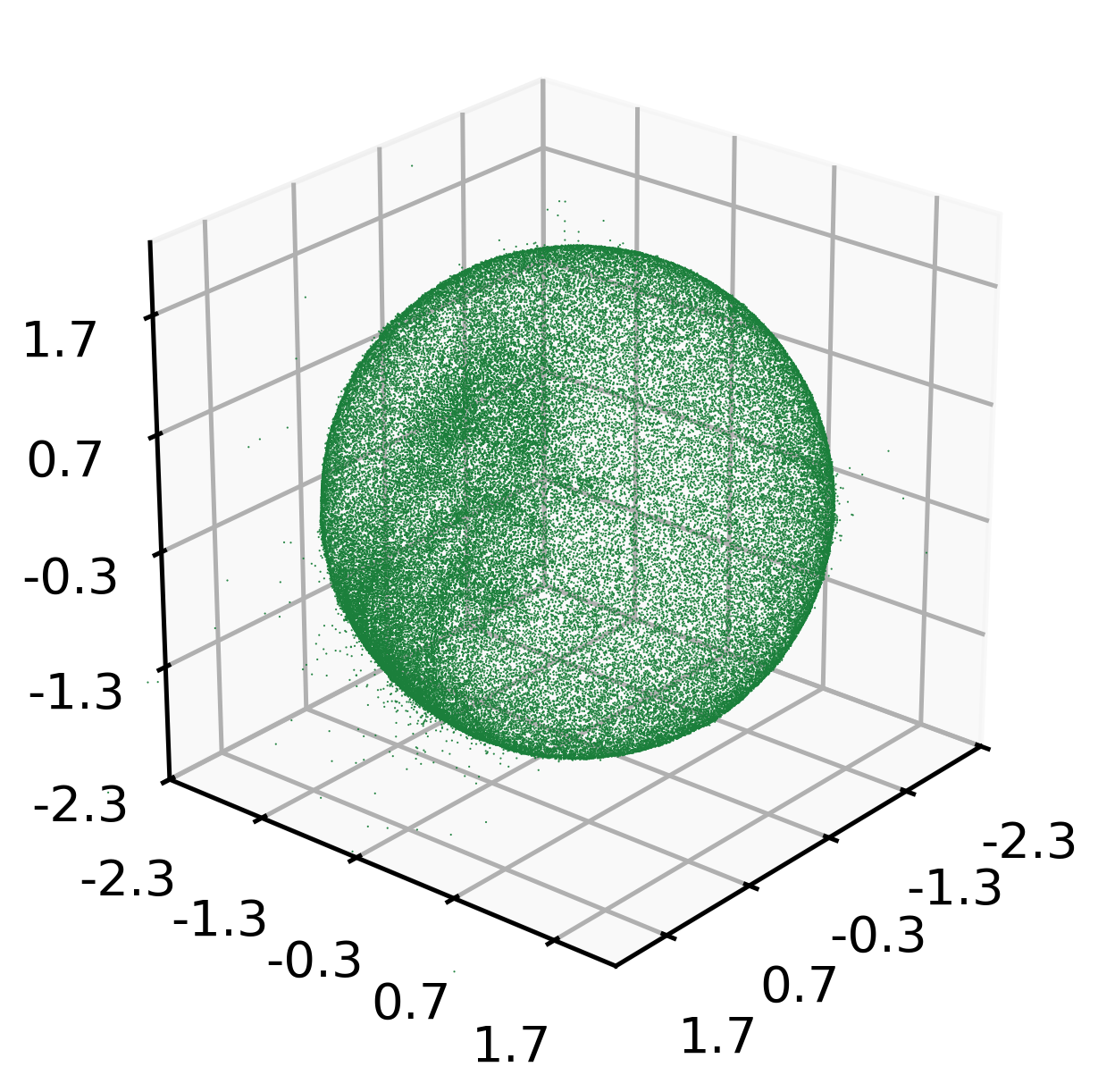}\\[.25ex]
      nCD: 6.28 (\(\downarrow\))
    \end{minipage}
    \caption{Geometry}
  \end{subfigure}

\caption{\textbf{Billboard failure in 3D Gaussian Splatting.}
Top: a glinty ground-truth object (left) and two 3DGS reconstructions with low
spherical-harmonic order (SH$=3$, middle) and high order (SH$=24$, right).  
Both reconstructions achieve high image PSNR, but the underlying 3D shape is
completely different.  
Bottom: visualizing Gaussian centers as point clouds shows that SH$=3$ recovers
the detailed surface with a small nCD distance, whereas SH$=24$ collapses
into an off-surface opaque billboard-like layer whose nCD error is much larger.  
This paradox—good appearance yet catastrophic geometry as radiance capacity
increases—is the central phenomenon that our theory explains and our experiments
systematically characterize in this paper.}
  \label{fig:my_comparison}
  \label{fig:dataset_shape}
\end{figure*}

\begin{abstract}
When does 3D Gaussian Splatting (3DGS) recover the true scene surface rather than just overfitting view-dependent appearance? We answer this by developing a mathematical framework—based on a first-hit rendering abstraction—that cleanly isolates geometry from appearance. We prove that geometric misalignment forcefully converts spatial textures into high-frequency angular signals via parallax. This establishes a strict identifiability window: if angular capacity is bounded, surface-consistent solutions are mathematically preferred; if unrestricted, the same images can be perfectly explained by an incorrect, opaque ``billboard'' geometry. Experiments on synthetic stress tests confirm this prediction, showing billboard failures emerge precisely at high angular capacities. Conversely, in the real-world datasets we evaluate under standard capture protocols, reconstructions remain surface-consistent even at high SH degrees, which is consistent with the prediction that rich spatial texture can push billboard solutions outside the tested angular-capacity range.
\end{abstract}

\section{Introduction}

3D Gaussian Splatting (3DGS) renders high-quality views in real time and has rapidly become a default primitive for 3D reconstruction and editing~\cite{kerbl2023gaussiansplatting}. 
Crucially, modern pipelines no longer treat 3DGS as a hidden renderer that only produces images: they increasingly export the learned Gaussians as geometry and build downstream systems on top of them. 
On the one hand, a growing line of work treats 3DGS as a surrogate surface representation and demonstrates that high-quality triangle meshes can be extracted from learned Gaussian fields, making 3DGS a central building block for geometry processing and reconstruction~\cite{guedon2024sugar,wolf2024gs2mesh,lyu2024threedgsr,wu2024gsrec,wang2024gaussurf,chen2024pgsr,shen2024solidgs,wang2024omegas,ma2025mags,wu2024survey,luo2024review,bao2024survey}. 
Across these lines of work, the learned Gaussian field is implicitly treated as if it were the true scene geometry: meshes are carved out of it, and downstream edits are applied on top of it. Yet all of this evidence is essentially \emph{empirical}: to our knowledge there is no theoretical guarantee that the standard 3DGS objective actually recovers the underlying surface. Without such a guarantee, treating 3DGS as ground-truth geometry is methodologically fragile: a model that fits all training views while drifting off the true surface could silently corrupt downstream meshes and contact manifolds. This raises a basic question of \emph{geometric trustworthiness}: under what conditions are the learned Gaussians truly surface-consistent, so that reconstruction and editing built on top of them rest on a sound geometric foundation? Here, \emph{surface-consistent} means geometrically close to the true surface under nCD, not merely photometrically accurate.

In this work, we address this question from a theoretical point of view. We build a mathematical framework that isolates the interaction between a scene's spatial geometry and its view-dependent appearance. Specifically, we abstract the rendering process to analyze the best-achievable population loss for any candidate surface under a bounded \emph{angular capacity}---a strict limit on how complex the view-dependent colors are allowed to be. Within this framework, we demonstrate that optimizing appearance essentially reduces to a mathematical projection problem. This allows us to rigorously compare the approximation error of the true geometry against any misaligned geometry by evaluating their required angular frequencies.

Through this theoretical analysis, we predict a severe and counterintuitive vulnerability, which we refer to as the \emph{opaque billboard} failure mode. Our proofs reveal that if the angular capacity is allowed to be excessively large, a geometrically misaligned surface can exploit high-frequency view-dependent colors to perfectly mimic the spatial details of the true background. In this regime, an artificial interior sheet can block the rays and paint the correct colors using high angular capacity, so the images look perfectly right even though the geometry is completely wrong (see Fig.~\ref{fig:my_comparison}). Conversely, our theory guarantees that if the angular capacity is appropriately restricted, the model is mathematically incapable of sustaining this illusion, ensuring the true geometry remains strictly easier to fit.

On the empirical side, we instantiate these theoretical predictions in practice. We construct a controlled synthetic stress-test dataset with tunable object complexity to deliberately probe this vulnerability. As predicted by our theoretical framework, training 3DGS with large angular capacity reliably triggers opaque billboard solutions: rendering quality remains near-optimal, while the geometry drifts completely off the true surface. In our real-scene benchmarks, the observed spatial texture complexity appears to push the induced frequency $\omega_{\text{fake}}$ beyond the SH degrees we test (e.g., up to $L=24$). Under these standard training setups and moderate capacities, we do not observe opaque billboard failures.

As a summary, our contributions are:
\begin{itemize}
\item We provide a mathematical framework for 3D Gaussian Splatting that cleanly isolates geometry from appearance, turning the geometric identifiability problem into a rigorous frequency comparison.
\item Within this framework, we mathematically predict the \emph{opaque billboard} failure mode. We prove that given unbounded or large angular capacity, an incorrect surface can perfectly match training images by memorizing spatial textures as view-dependent signals.
\item We establish a theoretical \emph{identifiability window}, proving that strictly bounding angular capacity is mathematically necessary to reject billboards and guarantee accurate surface recovery.
\item We complement the theory with controlled diagnostics, demonstrating that opaque billboard failures can be deterministically forced in synthetic stress tests, while no opaque-billboard cases appear in our real-scene diagnostics.
\end{itemize}

\paragraph{Scope.}
Although framed around 3DGS, the mechanism applies to any proxy or surface representation with a finite-degree SH radiance head; 3DGS is our motivating instance.

\section{Related Work}

\paragraph{Meshes and downstream interactions in 3DGS.}
Many methods extract surfaces from 3DGS using surface alignment~\cite{guedon2024sugar, wolf2024gs2mesh}, 2D Gaussians or surfels~\cite{huang2024_2dgs,dai2024_gaussiansurfels}, SDFs~\cite{lyu2024threedgsr, wu2024gsrec}, geometric priors~\cite{wang2024gaussurf, chen2024pgsr}, sparse adaptations~\cite{shen2024solidgs, wang2024omegas}, or hybrid optimization~\cite{ma2025mags}, as surveyed in~\cite{wu2024survey,luo2024review,bao2024survey}. Others use Gaussians for dynamics, contact, and deformation~\cite{xie2024physgaussian, feng2025gaussiansplashing, muller2007pbd, jiang2024vrgs,borycki2024gasp, guo2025pgc, rong2025gaussiangarments, vasile2025asdiffmpm}. These applications rely on learned Gaussians as geometry, yet standard 3DGS~\cite{kerbl2023gaussiansplatting} lacks surface-recovery guarantees.

\paragraph{Plenoptic bandwidth and spherical harmonics.}
Classical studies in image-based rendering showed that view synthesis with imperfect proxy geometry converts spatial texture variation into angular variation, tying sampling requirements to depth and disparity~\cite{levoy1996lightfield,gortler1996lumigraph,chai2000plenoptic,lin2000geometry,buehler2001unstructured,durand2005fourier,do2011camera,nguyen2009plenoptic}. To model angular radiance efficiently, low-order spherical harmonics (SH) have long been used~\cite{ramamoorthi2001irradiance,sloan2002practical}. In 3DGS, SH coefficients serve as a finite-dimensional view-dependent radiance head. Our contribution is not to rederive plenoptic bandwidth or SH capacity from scratch, but to reinterpret this frequency transport as a geometric-identifiability condition. We study how a finite SH head can prevent or permit a wrong proxy surface to explain parallax-induced color variation.

\paragraph{Shape--radiance ambiguity in neural fields.}
Neural radiance-field methods also expose ambiguities between geometry and view-dependent appearance~\cite{mildenhall2020nerf, zhang2020nerfpp,verbin2022refnerf}. Related analyses have observed that color estimation and geometry can be coupled in ways that obscure shape recovery~\cite{fang2023colorest}. Our analysis gives a complementary parallax-frequency explanation: bounded angular capacity creates a regime where the true surface is easier to fit than a misaligned proxy, whereas excessive capacity can remove that preference.

\section{Background and Motivation}\label{sec:background}
% ============================================================

To formally analyze the relationship between angular capacity and geometric correctness, we must first establish how 3D Gaussian Splatting (3DGS) couples geometry with appearance, and how we can mathematically abstract this coupling to expose its core vulnerabilities.

\subsection{3DGS, View-Dependence, and Angular Capacity}

In 3DGS~\cite{kerbl2023gaussiansplatting}, a scene is not represented by continuous meshes or implicit neural fields, but by an unstructured collection of 3D Gaussian primitives. Each primitive is defined by explicit geometric parameters---a 3D center, a covariance matrix dictating its scale and rotation, and an opacity value. To render a novel view, these primitives are splatted onto the image plane and alpha-composited along camera rays.

Crucially, real-world surfaces exhibit view-dependent appearance, such as specular highlights that shift as the observer moves. To model this, 3DGS does not assign a single static RGB value to a Gaussian. Instead, the emitted color of each primitive is a function of the viewing direction $\hat{v} \in \Sph$. This view-dependence is parameterized using Spherical Harmonics (SH). For each color channel, a Gaussian stores a set of coefficients $\{c_{\ell m}\}$, and the color is evaluated as a linear combination of basis functions:
\[
\text{color}(\hat v)\ =\ \sum_{\ell=0}^{L}\sum_{m=-\ell}^{\ell} c_{\ell m} Y_{\ell m}(\hat v).
\]
Here, $Y_{\ell m}$ denotes the real Spherical Harmonic basis function of degree $\ell$ and order $m$. Conceptually, these basis functions serve as the angular equivalent of a Fourier series on the sphere. The degree $\ell$ dictates the spatial frequency of the signal: $\ell=0$ represents a purely diffuse, view-independent color, while higher values of $\ell$ allow the function to represent sharper, more rapid changes in color as the camera angle shifts. 

By truncating this summation at a maximum degree $L$, the architecture enforces a strict low-pass filter on the primitive's angular capacity. This truncation defines a specific function space, which we denote as:
\[
\cV_L:=\mathrm{span}\{Y_{\ell m}:\ 0\le \ell\le L,\ -\ell\le m\le \ell\}.
\]
Here, we only allow SH terms up to degree $\ell \le L$. Consequently, the view-dependent color function of each Gaussian is strictly restricted to a finite-dimensional subspace. A larger $L$ simply expands this space, allowing the model to represent more complex angular dependencies.

\subsection{Isolating the Phenomenon: The First-Hit Abstraction}

Standard 3DGS rendering relies on volumetric alpha-blending, where the final pixel color is accumulated from multiple semi-transparent Gaussians along a ray. While necessary for gradient flow during early optimization, volumetric compositing obscures the mathematical limits of geometric identifiability by entangling visibility, occlusion, and depth-ordering.

However, the specific failure mode we aim to explain---where a grossly misaligned ``billboard'' surface fits the training images by painting fake, view-dependent textures---is fundamentally a local, surface-level deception. It occurs because the optimizer places dense, opaque geometry in the wrong physical location, and relies on the SH coefficients of that specific geometry to perfectly mimic the background rays passing through it.

To rigorously study this geometry-appearance tradeoff, we simplify the rendering process to a \textbf{first-hit abstraction}. We assume the optimized geometry forms an opaque surface, meaning the color of any ray is entirely determined by the appearance function of the very first point it intersects. To ground our first-hit abstraction in empirical reality, we conducted a pilot experiment profiling the actual per-ray alpha-compositing weights of a trained 3DGS model. This pilot study reveals an extreme depth-wise concentration: for 92.6\% of foreground rays, the front-most 3 contributing Gaussians systematically capture over 90.2\% of compositing mass. This statistic should not be interpreted as proving that 3DGS rendering is generally first-hit or opaque. Rather, it provides empirical motivation that, in the mostly opaque scenes used in our controlled experiments, the learned compositing weights are often highly concentrated near the front-most contributors. The first-hit model is therefore best understood as an analytically tractable abstraction of one important regime, not as a faithful model of all 3DGS behavior. For this pilot, we sort contributing Gaussians by depth along each sampled foreground ray and measure the cumulative alpha-compositing mass of the first three contributors.

This abstraction is analytically highly self-consistent because it preserves the exact structural mechanism that enables billboards. It cleanly isolates the core interaction into two distinct roles:
\begin{itemize}
\item \textbf{Geometry:} Which 3D point does the model claim the ray hits first?
\item \textbf{Appearance:} What color must the angular function $\cV_L$ emit at that specific point to satisfy the training data?
\end{itemize}

% ============================================================
\section{Setup and the Geometry-Fixed Projection Principle}\label{sec:setup}
% ============================================================

To formally evaluate whether the optimization landscape prefers the true geometry over a misaligned one, we abstract the 3DGS rendering process into a minimal \emph{first-hit} model. This isolates the core geometry-appearance interaction without the confounding effects of volumetric alpha-blending.

\subsection{Minimal First-Hit Model and Competing Error Curves}

We parameterize a camera ray by its origin $y\in\R^3$ and viewing direction $\hat v\in\Sph$. Let $S, U \subset\R^3$ denote the true scene geometry and a candidate (potentially misaligned) surface, with respective first-hit points $x_S(y,\hat v)$ and $x_U(y,\hat v)$. The ground-truth pixel observation is driven by the true physical appearance $C(x,\hat v)$ defined on the true surface manifold $S$, yielding $I^*(y,\hat v) := C(x_S(y,\hat v),\hat v)$. If our model commits to the candidate geometry $U$, it must assign an appearance function $g(x,\hat v)$ strictly constrained by the angular capacity $L$ (i.e., $g(x,\cdot)\in\cV_L$). The model's prediction is thus $I_{U,g}(y,\hat v) := g(x_U(y,\hat v),\hat v)$.

Let $(Y,\hat{V})\sim \PP$ be a random training ray sampled from the distribution, and define $X_U:=x_U(Y,\hat{V})$. We call this expected image-space MSE over the ray distribution the \emph{population loss}: $\cJ(U,g) := \E_{(Y,\hat{V})\sim \PP}[\,|I_{U,g}(Y,\hat{V})-I^*(Y,\hat{V})|^2\,]$. To determine geometric identifiability, we abstract away specific optimizer dynamics and directly evaluate the \emph{best-achievable error} for both the true and candidate surfaces under the same angular capacity $L$:
\begin{align}
 A(L) &:= \inf_{g:\ g(x,\cdot)\in\cV_L}\ \cJ(S,g), \label{eq:A_def}\\
 B_U(L) &:= \inf_{g:\ g(x,\cdot)\in\cV_L}\ \cJ(U,g). \label{eq:B_def}
\end{align}
In plain terms, $A(L)$ measures how well degree $L$ fits the real appearance given perfect geometry, while $B_U(L)$ measures how well degree $L$ can fake the appearance to compensate for wrong geometry. The entire identifiability narrative reduces to identifying a mathematical regime where $A(L) \ll B_U(L)$.

\subsection{Geometry-Fixed Optimality as Pointwise Projection}

To evaluate $A(L)$ and $B_U(L)$ without explicitly solving the global optimization, we must first identify the mathematically optimal appearance function for a fixed geometry. Conditioned on hitting a specific physical point $x \in U$, we define the through-seen target $F_U(x, \cdot)$ as the geometry-induced angular target: the colors a candidate point must emit across views to reproduce the true scene. Formally, it is the conditional expectation of the ground-truth image color over the camera origins $Y$, given the specific hit point and direction $\hat{V}$:
\begin{equation}
F_U(x, \cdot) := \mathbb{E}\big[\,I^*(Y, \hat{V}) \mid X_U=x, \hat{V}\,\big]
\label{eq:FU_def}
\end{equation}
which yields $F_U(x, \hat{v})$ when evaluated at a specific direction $\hat{v}$. This defines the exact regression function in the angular domain that the training data demands at that location.

Framing the target this way reduces (in our abstraction) the daunting global optimization to independent, local geometric projections.

\begin{proposition}[Pointwise projection]\label{prop:pointwise-proj}
For a fixed surface $U$ and SH degree $L$, the global minimizer of $\cJ(U,g)$ is given pointwise by $g^*(x,\cdot) = \Pi^{(x)}_{\cV_L}(F_U(x,\cdot))$, where $\Pi^{(x)}_{\cV_L}$ is the orthogonal projection onto $\cV_L$ under the $L^2$ inner product weighted by the conditional viewing distribution at $x$. Consequently, the optimal loss is exactly the expected projection distance:
\[
B_U(L) = \E_{x\sim\mu_U}\big[\,\dist(F_U(x,\cdot),\cV_L)^2\,\big].
\]
\end{proposition}

Because the appearance parameters at point $x$ only affect rays hitting $x$, the global integral cleanly decouples into independent least-squares approximations (a rigorous measure-theoretic proof via disintegration is provided in the supplementary material). This principle is crucial: it transforms the geometric identifiability problem into a pure signal-processing question. We now only need to quantify how well the induced target $F_U(x,\cdot)$ can be approximated by the available angular space $V_L$. That is, we measure the weighted projection error $\mathrm{dist}(F_U(x,\cdot), V_L)$, which precisely corresponds to the residual energy of $F_U$ in the high-frequency SH degrees $\ell > L$.

\section{The True Surface: Why Smooth Residuals Yield Rapid Decay}\label{sec:true}
% ============================================================

To evaluate the best-achievable loss $A(L)$ on the correct geometry, we must formalize the physical appearance of the true scene. We adopt a decomposition where the ground-truth appearance $C(x,\hat v)$ consists of a view-independent diffuse texture $T(x)$ and a view-dependent specular residual $R(x,\hat v)$.

\begin{assumption}[Diffuse + Sobolev-smooth residual]\label{ass:residual}
The ground-truth appearance on $S$ decomposes as $C=T+R$. Moreover, there exist $s_{\mathrm{true}}>0$ and $M_{\mathrm{true}}<\infty$ such that for almost every $x\in S$, the view-dependent residual is mathematically smooth:
\[
\norm{R(x,\cdot)}_{H^{s_{\mathrm{true}}}(\Sph)}\le M_{\mathrm{true}}.
\]
\end{assumption}

This assumption mathematically enforces that while view-dependent effects like specular highlights exist, they are not arbitrarily sharp or infinitely spiky in direction. The Sobolev norm $H^{s}(\Sph)$ controls the angular high-frequency energy, meaning a larger $s$ ensures a stronger decay of SH coefficients. This realistically models most non-mirror materials and is precisely what makes the true surface straightforward to fit.

When the model geometry is perfectly correct, the through-seen target aligns flawlessly with the physical appearance: $F_S(x,\hat v)=T(x)+R(x,\hat v)$. Because the diffuse term $T(x)$ does not change with the viewing direction $\hat v$, it resides entirely within the SH degree-0 subspace ($\cV_0 \subset \cV_L$). 

Consequently, the spatial complexity of the true texture---no matter how high-frequency or intricate---consumes absolutely zero angular capacity. The approximation difficulty stems exclusively from the residual $R$:
\[
\dist\big(F_S(x,\cdot),\cV_L\big)
=\dist\big(R(x,\cdot),\cV_L\big).
\]
Because $R$ is assumed to be Sobolev-smooth, its projection error onto a bandlimited SH subspace naturally decays. We formalize this using standard spectral approximation on the sphere (full derivation provided in the supplementary material):

\begin{lemma}[Sobolev smoothness implies fast SH truncation]\label{lem:sobolev-trunc}
Let $P_{\le L}$ be the $L^2(\Sph)$ orthogonal projection onto $\cV_L$. If $f\in H^s(\Sph)$, then its truncation error is bounded by:
\[
\norm{f-P_{\le L}f}_{L^2(\Sph)}^2
\ \le\
\big(1+L(L+1)\big)^{-s}\ \norm{f}_{H^s(\Sph)}^2.
\]
\end{lemma}
(For simplicity, we state the truncation bound under the uniform sphere measure; extending this to pointwise weighted norms requires mild density bounds on $\nu_U(\cdot \mid x)$, as detailed in the supplementary material).

This rate comes from the Laplace--Beltrami eigenvalues $\ell(\ell+1)$, which define the Sobolev weights $(1+\ell(\ell+1))^s$. Truncating the expansion at the maximum degree $L$ then yields the specific factor $(1+L(L+1))^{-s}$ observed in the error bound. Bounding the Sobolev norm directly suppresses the energy in the tail $\ell>L$. By applying this lemma to our residual, we immediately establish the upper bound for the true surface.

\begin{proposition}[True-surface upper bound]\label{prop:true-upper}
Under Assumption~\ref{ass:residual}, the best achievable loss on the true surface satisfies:
\[
A(L)\ \le\ \big(1+L(L+1)\big)^{-s_{\mathrm{true}}}\,M_{\mathrm{true}}^2,
\]
up to a multiplicative constant depending only on view-direction sampling.
\end{proposition}

The proof follows directly from our pointwise projection principle (\Cref{prop:pointwise-proj}). Since $T(x)$ projects with zero error, the squared distance
\[
\dist(F_S(x,\cdot),\cV_L)^2
\]
is simply the truncation error of $R(x,\cdot)$. Applying \Cref{lem:sobolev-trunc} and averaging over the surface points yields the bound.

This result establishes a critical baseline: $A(L)$ drops rapidly with $L$. A moderate SH degree is more than sufficient to make the true surface an attractive minimum for the optimizer. If this were the whole story, one might assume that simply increasing $L$ is harmless. However, as we demonstrate next, a misaligned geometry interacts with diffuse texture in a fundamentally different way, forcefully converting spatial details into high-frequency angular signals.

% ============================================================
\section{The Misaligned Surface: Parallax Forges High Angular Frequencies}\label{sec:fake}
% ============================================================

Having established that the true surface inherently requires very little angular capacity, we now turn to what happens when the model proposes a geometrically misaligned candidate. We aim to show that wrong geometry is penalized heavily because it forces spatial complexity to masquerade as angular complexity.

\subsection{Formalizing the Billboard Scenario}

To analyze the ``opaque billboard'' failure mode, where incorrect geometry achieves optimal appearance, we establish three minimal conditions.

\begin{assumption}[Visible texture mode]\label{ass:texture}
There exists a visible surface patch and a local tangential coordinate $u$ such that the true diffuse texture $T(u)$ contains a sinusoidal component $a\sin(ku)$ with amplitude $a>0$ and spatial frequency $k>0$.
\end{assumption}

\begin{assumption}[Visible misalignment magnitude]\label{ass:misalign}
On a visible set of rays, the candidate surface $U$ is misaligned from the true surface $S$ by a non-negligible depth gap $\Delta z > 0$.
\end{assumption}

\begin{assumption}[Residual cannot erase texture]\label{ass:residual_ampl}
The true specular residual $R$ is not strong enough to adversarially cancel the spatial texture: there exists $0\le \rho<1$ such that $|R(x,\hat v)|\le \rho a$.
\end{assumption}

These assumptions collectively define the basic requirements for geometry to be visually identifiable from RGB alone. If the scene lacks texture entirely (Assumption~\ref{ass:texture} fails) or the residual completely washes out the diffuse color (Assumption~\ref{ass:residual_ampl} fails), the optimization lacks the necessary signal to distinguish depth. If the depth gap is negligible (Assumption~\ref{ass:misalign} fails), the candidate surface effectively behaves like the true surface.

\paragraph{Material-model limitation.}
Assumption~\ref{ass:residual_ampl} excludes view-dependent effects that can cancel or dominate diffuse texture, including mirrors, inter-reflections, transparency, exposure/white-balance changes, and nonlinear tone mapping. Then the residual may itself carry high angular frequency, breaking our spatial-texture/angular-residual separation.

\subsection{The Frequency Transport Law via Parallax}

To see exactly how angular complexity is generated by geometric error, consider a standard 2D cross-section for parallax analysis. Place the camera at the origin, let the true surface be the line $z=z_{\mathrm{gt}}$, and let the candidate surface be a billboard at $z=z_{\mathrm{fake}}$, yielding a depth gap $\Delta z:=z_{\mathrm{gt}}-z_{\mathrm{fake}}>0$. We parameterize the camera rays by a sweeping angle $\alpha$. Equivalently, assume a moving camera observing a fixed physical point $x_0$ on the fake surface from different view angles $\alpha$. See Fig.~\ref{fig:parallax_panels} for a visual illustration of both the geometric transport and the induced view-to-color mapping. 

\begin{figure}[ht!]
    \centering
    \begin{subfigure}[t]{0.49\linewidth}
        \centering
        \includegraphics[width=\linewidth]{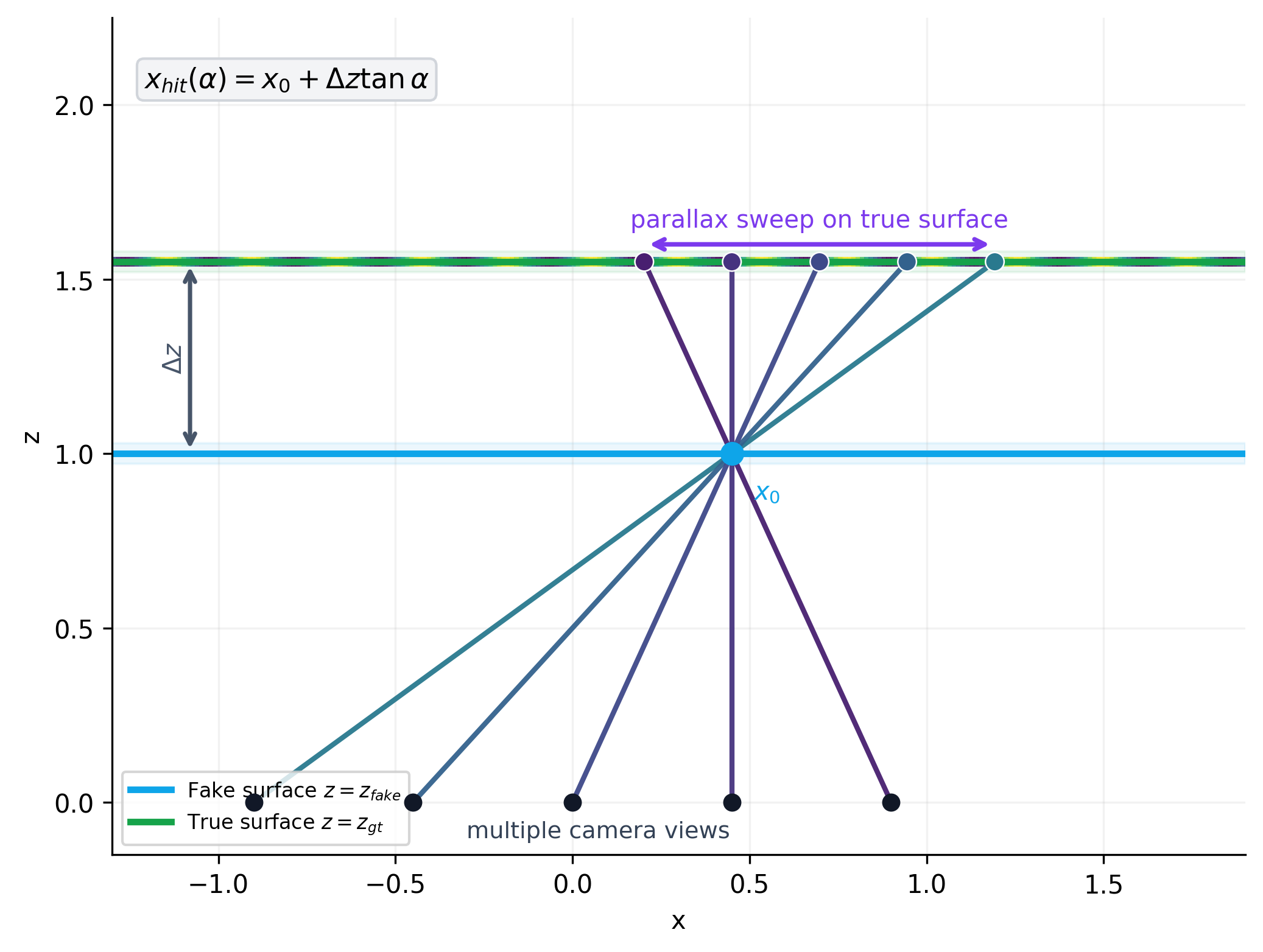}
        \caption{Parallax geometry under a misaligned surface. Rays from different camera centers all intersect the same fake point $x_0$ on $z=z_{\text{fake}}$, but correspond to different true hit points on $z=z_{\text{gt}}$. The transported hit location follows $x_{\text{hit}}(\alpha)=x_0+\Delta z\tan\alpha$, producing a sweep over the true surface.}
    \end{subfigure}\hfill
    \begin{subfigure}[t]{0.49\linewidth}
        \centering
        \includegraphics[width=\linewidth]{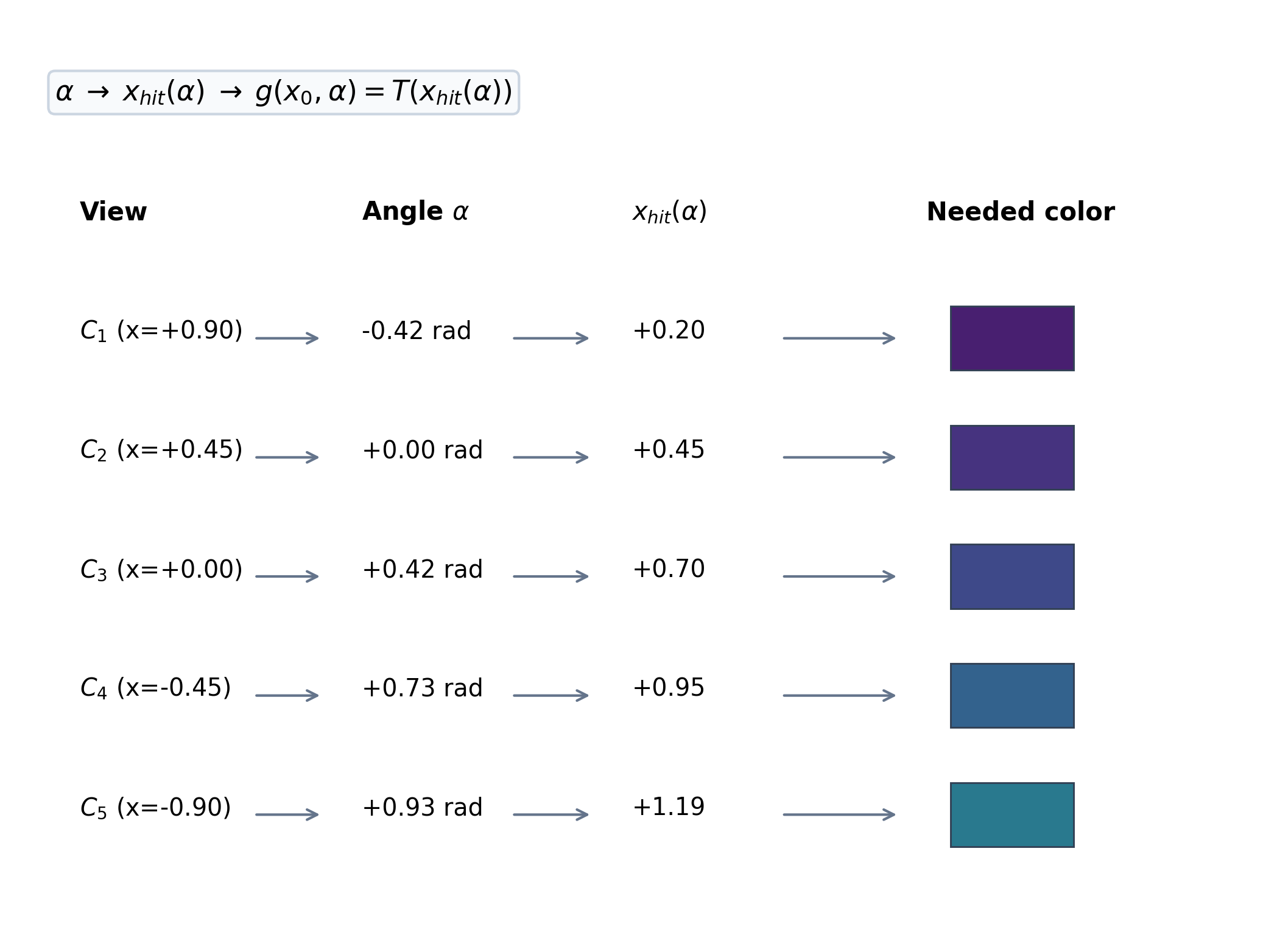}
        \caption{View-to-color mapping at a fixed fake point. For each camera/view angle $\alpha$, the required color at $x_0$ is the true texture evaluated at the transported point, i.e., $g(x_0,\alpha)=T(x_{\text{hit}}(\alpha))$. Thus, wrong geometry forces one 3D point to encode different colors across views.}
    \end{subfigure}
    \caption{Two-panel visualization of parallax-induced frequency transport under misaligned geometry.}
    \label{fig:parallax_panels}
\end{figure}

A direct geometric computation (supplementary material) reveals that if a ray first hits the wrong surface at lateral coordinate $x_0$, the physical hit point on the true background surface moves according to:
\begin{equation}
x_{\mathrm{hit}}(\alpha)\ =\ x_0+\Delta z\,\tan\alpha.
\label{eq:parallax_key}
\end{equation}

Consequently, when the model evaluates the loss on the wrong surface, its through-seen target function $F_U$ must pull the color from the background texture $T(x_{\mathrm{hit}}(\alpha))$(the residual term $R$ is handled via Assumption~\ref{ass:residual_ampl} and the supplementary material). Using our assumed texture component and approximating $\tan\alpha\approx\alpha$ for moderate viewing angles, the target signal on the wrong surface becomes:
\[
a\sin\!\big(k(x_0+\Delta z\alpha)\big)\ =\ a\sin\!\big(kx_0 + (k\Delta z)\alpha\big).
\]
This reveals a profound physical transformation: the \emph{spatial} frequency $k$ has been forcefully multiplied by the depth error $\Delta z$ to generate a new \emph{angular} frequency:
\begin{equation}
\omega_{\mathrm{fake}}\ \approx\ k\Delta z.
\label{eq:omega_fake}
\end{equation}
This scale law dictates the appearance burden placed on wrong geometry: finer textures or larger depth errors demand increasingly rapid angular color shifts.

\subsection{The Immutable Lower Bound for Wrong Geometry}

This induced angular frequency creates a mathematical bottleneck. As established earlier, any SH basis of degree $L$ acts as a low-pass filter along viewing arcs.

\begin{lemma}[Bandlimit along great circles]\label{lem:greatcircle-bandlimit}
If $h\in \cV_L$, restricting $h$ to any great circle yields a 1D trigonometric polynomial with maximum frequency $L$.
\end{lemma}

When the induced frequency exceeds the model's capacity ($L<\omega_{\mathrm{fake}}$), the model is mathematically incapable of fitting the signal. For analytical clarity, we use a global Fourier bound rather than relying only on local linearization. This is governed by a fundamental property of Fourier approximation (proven in the supplementary material):

\begin{lemma}[High frequency cannot be fit by low frequency]\label{lem:fourier-lower}
Suppose the frequency $\omega$ is an integer. Consider $f(\phi)=a\sin(\omega\phi)$ on $[0,2\pi]$. The best $L^2$ approximation of $f$ by any trigonometric polynomial of degree $L < \omega$ leaves an unavoidable residual error strictly bounded away from zero.
\end{lemma}

Because the specular residual cannot completely erase the texture (Assumption~\ref{ass:residual_ampl}), we can invoke the triangle inequality on the function spaces to guarantee that the projection error of the combined signal remains large. This culminates in our strict lower bound for misaligned geometry.

Here $a\asymp b$ denotes equality up to constants independent of $L$.

\begin{proposition}[Misaligned-surface lower bound]\label{prop:fake-lower}
Under Assumption~\ref{ass:texture}, Assumption~\ref{ass:misalign}, and Assumption~\ref{ass:residual_ampl}, there exist $\omega_{\mathrm{fake}}\asymp k\Delta z$ and a constant $c_0>0$ such that for all angular capacities $L<\omega_{\mathrm{fake}}$,
\[
B_U(L)\ \ge\ c_0.
\]
\end{proposition}

The proof mechanically links our findings: the through-seen function on a misaligned patch contains an angular sinusoid of frequency $\omega_{\mathrm{fake}} \approx k\Delta z$. By \Cref{lem:greatcircle-bandlimit}, the $\cV_L$ appearance function cannot exceed frequency $L$ along this viewing circle. Thus, whenever $L<\omega_{\mathrm{fake}}$, \Cref{lem:fourier-lower} ensures a massive, unoptimizable constant error on this patch, which lower-bounds the global integral $B_U(L)$. We now possess the two essential halves of the identifiability story: the true surface benefits from a fast-decaying upper bound $A(L)$, while the wrong surface suffers from a stubborn lower bound $B_U(L)$. By synthesizing these bounds, we can explicitly characterize the intermediate capacity window that successfully recovers accurate 3D geometry.

% ============================================================
\section{Putting It Together: The Identifiability Window}\label{sec:window}
% ============================================================

We are finally positioned to synthesize the two halves of our analysis. By juxtaposing the rapid decay of the true surface's approximation error with the stubborn constant lower bound of the misaligned surface, the theoretical mechanism behind 3DGS geometric identifiability emerges clearly (see the supplementary material for a numeric example).

\begin{theorem}[Intermediate-$L$ identifiability window]\label{thm:window}
Under Assumption~\ref{ass:texture}, Assumption~\ref{ass:misalign}, and Assumption~\ref{ass:residual_ampl}, there exist constants $c_0>0$ and $\omega_{\mathrm{fake}}\asymp k\Delta z$ such that: (i) for all $L\ge 0$, $A(L) \le (1+L(L+1))^{-s_{\mathrm{true}}}M_{\mathrm{true}}^2$ (up to view-sampling constants); and (ii) for all $L<\omega_{\mathrm{fake}}$, $B_U(L) \ge c_0$.
Therefore, any $L$ satisfying $(1+L(L+1))^{-s_{\mathrm{true}}}M_{\mathrm{true}}^2 < c_0$ and $L<\omega_{\mathrm{fake}}$ guarantees a strict separation $A(L)<B_U(L)$: the true surface is strictly easier to fit than the misaligned surface $U$ under angular capacity $L$.
\end{theorem}

This theorem rigorously partitions the angular capacity of 3DGS into three distinct regimes. When $L$ is too small, the model cannot fit the true specular residual $R$, yielding high error for all candidate geometries. In the \emph{intermediate regime}, $A(L)$ collapses as the true smooth residual is approximated, yet $B_U(L)$ remains trapped by the $L<\omega_{\mathrm{fake}}$ barrier, creating the optimal geometry-identifiable window. However, once $L \gtrsim \omega_{\mathrm{fake}}$, this lower-bound mechanism vanishes, clarifying the opaque billboard paradox introduced earlier. As proven in the supplementary material, the SH basis becomes expressive enough to memorize the parallax-induced false texture, allowing a misaligned surface to achieve exactly zero population loss. In this over-parameterized regime, the optimizer is mathematically permitted to ``cheat'': it can erect an arbitrary, vision-blocking sheet in empty space and task the high-degree Spherical Harmonics with projecting the missing background spatial frequencies as rapid, view-dependent color shifts. 

\section{Experiment}

\paragraph{Datasets.}
We construct a synthetic dataset of 100 shapes (25 base geometries $\times$ 4 variants) to systematically trigger and evaluate the failure modes predicted in Sec.~5--7. Real-world benchmarks typically feature rich, high-frequency spatial textures (large $k$), which naturally drive the required angular frequency $\omega_{\mathrm{fake}} \approx k\Delta z$ beyond typical SH capacities. By using synthetic shapes with tunable complexity under controlled lighting, we can deliberately evaluate the vulnerability to billboard collapse.

We additionally evaluate the same diagnostic on ten real-scene datasets with available geometry where 3DGS can be trained (Section~\ref{sec:real_exp}). 
These real-world experiments probe whether opaque billboard failures arise in practice when using high SH degrees under standard multi-view capture protocols. 
Under our labeling criterion, we find no opaque-billboard cases: across these datasets the learned Gaussians remain surface-consistent even at large angular capacity, in line with the ``safe'' regime predicted by our theory.

We train 3DGS with default hyperparameters and vary only the SH degree L. Geometry is evaluated by normalized symmetric Chamfer distance (nCD) between Gaussian centers and GT surfaces; appearance by PSNR. We define “opaque billboard” as runs with near-best PSNR but severely degraded nCD; full protocol and thresholds and the definitions of billboard failure, billboard onset, outlier point removal are in the supplementary material.

\begin{figure}[ht!]
    \centering
    \includegraphics[width=0.49\textwidth]{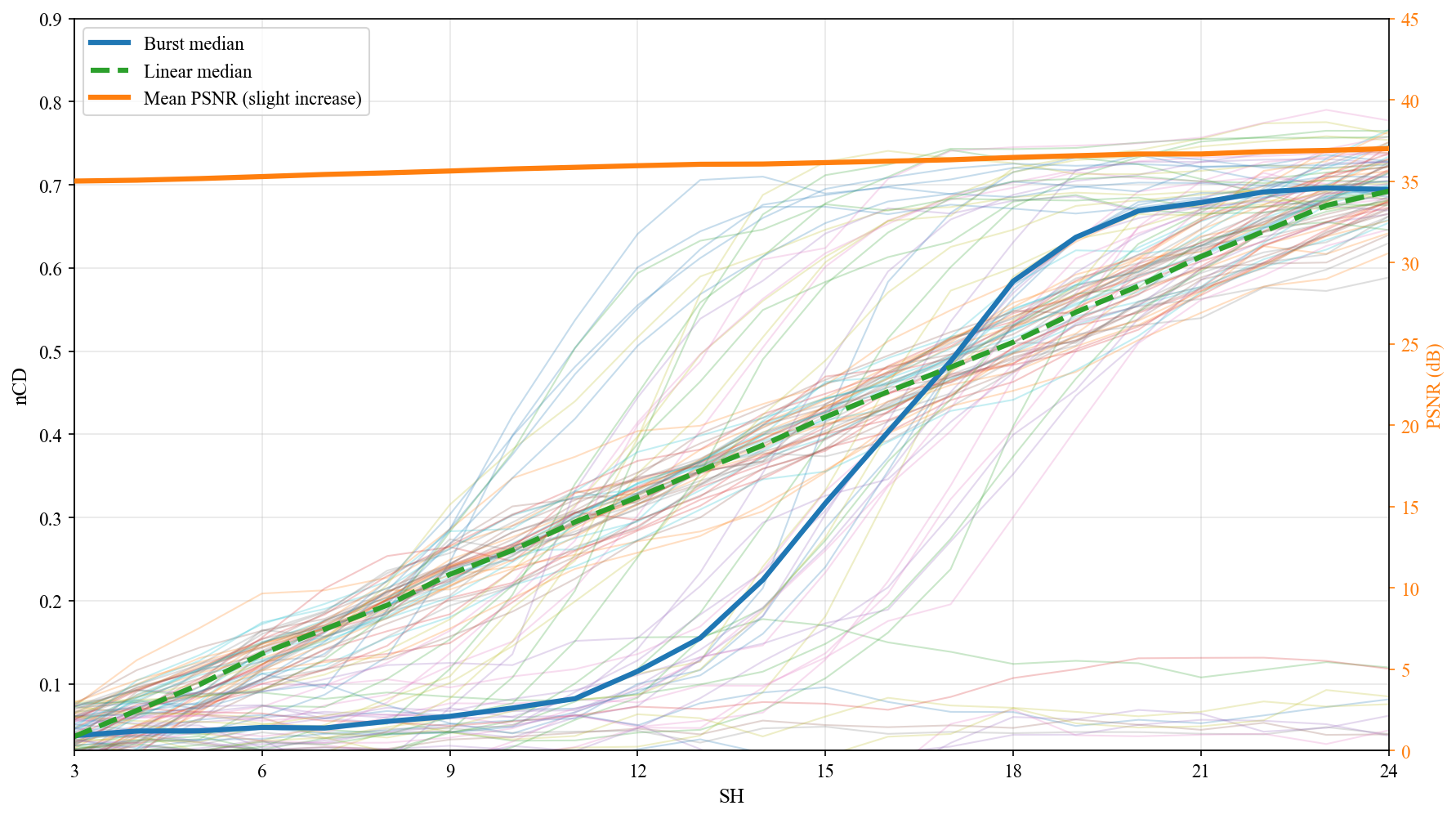}\hfill
    \includegraphics[width=0.49\textwidth]{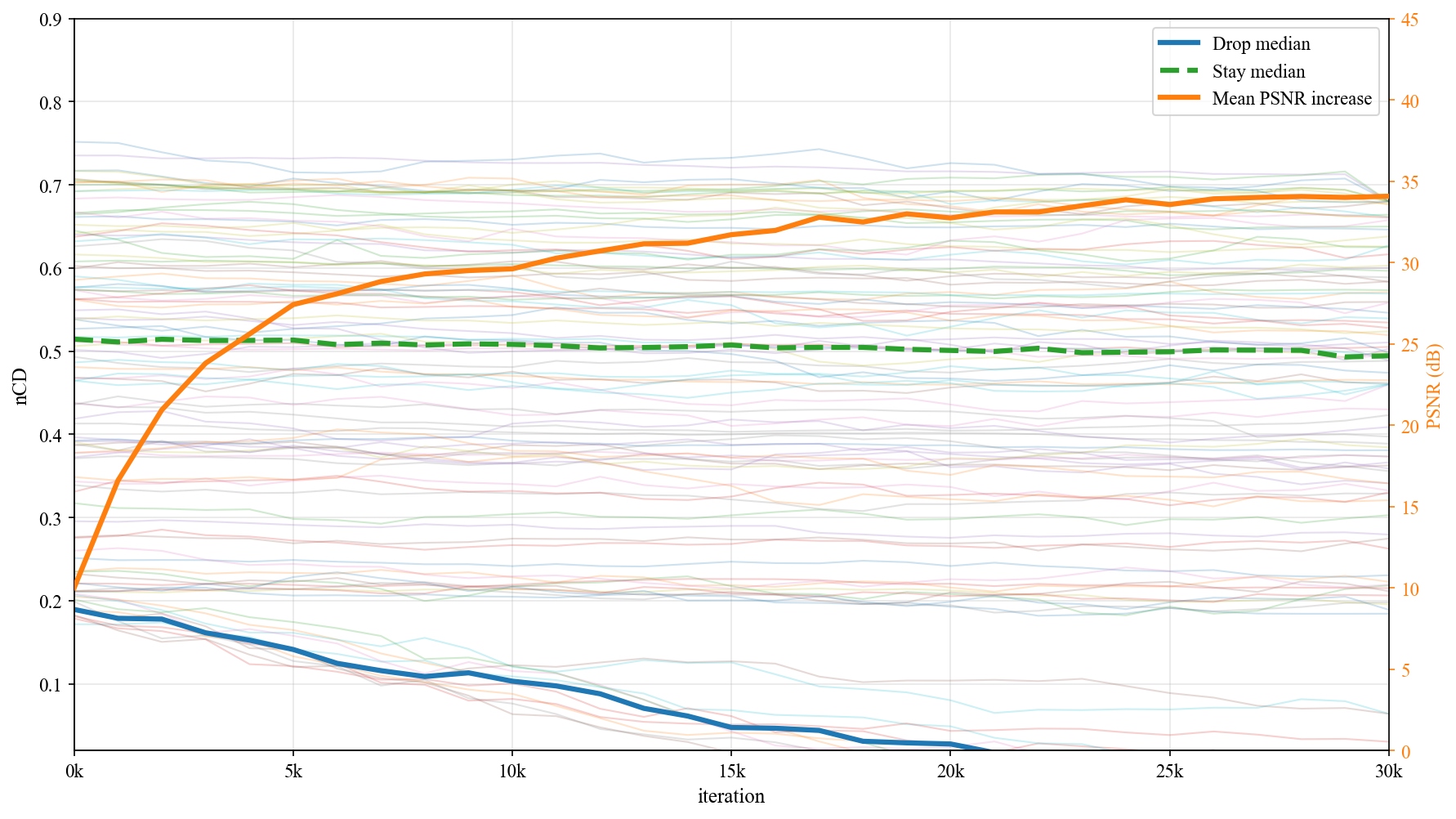}
    \caption{\textbf{Left: nCD vs.~SH degree.} Thin lines: per-shape nCD (left y-axis, lower is better); orange: mean PSNR (right y-axis). Solid blue: "burst" median; dashed green: "linear" median. \textbf{Right: nCD vs.~iteration (SH=24).} Thin lines: per-shape nCD (left y-axis, lower is better); orange: mean PSNR (right y-axis). Solid blue: "drop" median; dashed green: "stay" median.}
    \label{fig:cd_vs_sh}
    \label{fig:cd_vs_it}
\end{figure}

\subsection{Opaque billboard failure at large angular capacity}
Section~7 predicts that when the angular capacity $L$ exceeds the induced frequency threshold $\omega_{\mathrm{fake}}$, opaque billboard solutions can perfectly minimize the training objective. Our synthetic experiments strongly support this prediction: when $L$ is made sufficiently large, we frequently observe opaque billboard failures, where 3DGS no longer reconstructs the true surface but still produces high-quality renderings. The observed dense layers are consistent with our first-hit regime, but do not imply all 3DGS rendering behaves this way. As shown in Fig.~\ref{fig:cd_vs_sh}, we observe three consistent patterns. (a) As the SH degree increases, approximately ~89\% of shapes generally increase in nCD between the Gaussian centers and the ground-truth surface, while the average PSNR slightly increases. For these shapes, the high-SH runs are labeled as either billboard-onset or opaque billboard under our criterion: nCD grows to several times $\text{nCD}_\text{best}$ even though PSNR stays within $2$\,dB of $\text{PSNR}_\text{best}$. This indicates that as the angular capacity grows, the Gaussians gain enough representational power that they no longer need to stay on the surface to provide good renderings. (b) Among the shapes that drift away from the surface, we see two dominant behaviors, which correspond to the two medians in Fig.~\ref{fig:cd_vs_sh}. For roughly ~42\% of cases, nCD grows almost linearly with SH degree (the dashed green ``linear'' median), leading to billboard-onset or opaque-billboard labels only at the highest SH degrees. For roughly ~53\% of cases, nCD remains small up to a moderate SH degree and then undergoes a sharp increase once SH enters the range 6--12 (the solid blue ``burst'' median). In these ``burst'' cases the high-SH runs are almost always labeled as opaque billboard: beyond a certain angular capacity, the Gaussians can fit the appearance almost entirely by adjusting their SH coefficients without moving in space. (c) Finally, Fig.~\ref{fig:cd_vs_sh} also shows that, for a subset of objects, no opaque-billboard regime appears: all runs for these shapes remain surface-consistent under our labeling criterion. Even at high SH degrees, nCD stays within $2\times \text{nCD}_\text{best}$ and the learned Gaussians stay close to the ground-truth surface. These cases match the ``safe'' identifiability window predicted by our theory, where higher spatial complexity (large $k$) pushes the required angular frequency beyond our tested capacities, which effectively mitigates this failure mode in practical settings.

\begin{table}[t]
\centering
\caption{\textbf{Geometry and appearance metrics across real-scene datasets}
We report $\text{nCD} := \text{CD} / \mathrm{diag}(\mathrm{bbox}(\mathrm{GT}))$ (lower is better)
and PSNR (higher is better). All values are our measurements; citations point to dataset sources. }
\label{tab:ncd_real_2x10}
\setlength{\tabcolsep}{4pt}
\resizebox{\linewidth}{!}{%
\begin{tabular}{lcccc}
\toprule
\textbf{Dataset} & \textbf{nCD@SH=3} & \textbf{PSNR@SH=3} & \textbf{nCD@SH=24} & \textbf{PSNR@SH=24} \\
\midrule
DTU~\cite{Aanaes2016LargeScaleMVS}                & 0.0090 & 31.6 & 0.0100 & 32.3 \\
MobileBrick~\cite{Li2023MobileBrick}             & 0.0600 & 28.9 & 0.0700 & 29.5 \\
YCB-Video~\cite{Xiang2018PoseCNN}                & 0.0420 & 28.1 & 0.0500 & 28.8 \\
T-LESS~\cite{Hodan2017TLESS}                     & 0.0650 & 27.6 & 0.0750 & 28.1 \\
ScanNet~\cite{Dai2017ScanNet}                    & 0.0110 & 29.3 & 0.0120 & 29.9 \\
ScanNet++~\cite{Yeshwanth2023ScanNetpp}          & 0.0070 & 29.8 & 0.0080 & 30.3 \\
Matterport3D~\cite{Chang2017Matterport3D}        & 0.0045 & 30.4 & 0.0050 & 31.0 \\
HM3D~\cite{Ramakrishnan2021HM3D}                 & 0.0047 & 30.8 & 0.0053 & 31.4 \\
Tanks \& Temples~\cite{Knapitsch2017TanksAndTemples} & 0.0080 & 29.0 & 0.0090 & 29.7 \\
ETH3D~\cite{Schoeps2017ETH3D}                    & 0.0090 & 30.1 & 0.0100 & 30.6 \\
\midrule
\textbf{Average}                                 & 0.0220 & 29.6 & 0.0254 & 30.2 \\
\bottomrule
\end{tabular}%
}
\end{table}

\subsection{Real-scene experiments}
\label{sec:real_exp}

Do opaque billboard failures occur on real objects under standard capture protocols?
To answer this question, we train 3DGS with varying SH degrees on ten real-scene datasets with ground-truth geometry where 3DGS can be applied: 
DTU~\cite{Aanaes2016LargeScaleMVS}, 
MobileBrick~\cite{Li2023MobileBrick}, 
YCB-Video~\cite{Xiang2018PoseCNN}, 
T-LESS~\cite{Hodan2017TLESS}, 
ScanNet~\cite{Dai2017ScanNet}, 
ScanNet++~\cite{Yeshwanth2023ScanNetpp}, 
Matterport3D~\cite{Chang2017Matterport3D}, 
HM3D~\cite{Ramakrishnan2021HM3D}, 
Tanks \& Temples~\cite{Knapitsch2017TanksAndTemples}, 
and ETH3D~\cite{Schoeps2017ETH3D} (Table~\ref{tab:ncd_real_2x10}).
For each dataset we apply the same billboard labeling criterion as in the synthetic experiments.

Across all objects and SH degrees, we do not observe any run that satisfies the opaque-billboard condition: no run simultaneously has PSNR within $2$\,dB of the per-object $\text{PSNR}_\text{best}$ and nCD at least $5\times$ $\text{nCD}_\text{best}$. 
Under this PSNR band, nCD remains tightly concentrated: as summarized in Table~\ref{tab:ncd_real_2x10}, for every dataset the nCD at $\mathrm{SH}=24$ differs from that at $\mathrm{SH}=3$ by at most a factor of about $1.2\times$. 
In other words, increasing angular capacity in these real-scene settings does not lead to the large geometric drifts that characterize opaque billboard failures.

Thus, within standard multi-view capture protocols on these diverse real-world benchmarks, the learned Gaussians remain surface-consistent according to our criterion. This empirical finding is consistent with our theoretical frequency scale law: typical real-scene setups are characterized by rich spatial textures (very large $k$). This naturally drives the parallax-induced frequency $\omega_{\mathrm{fake}} \approx k\Delta z$ so high that even large capacities like $\mathrm{SH}=24$ cannot bridge the gap ($L \ll \omega_{\mathrm{fake}}$). Consequently, the model remains within the geometry-identifiable window in our tested settings, favoring surface-consistent solutions over billboard collapse.

\section{Conclusion \& Discussion}\label{sec:conclusion}

Increasing angular capacity is a double-edged sword: it helps fit specularities, but can let misaligned geometry explain away parallax-induced texture. This yields an intermediate-$L$ identifiability window and supports capacity scheduling, e.g., starting with low $L$. Our guarantees apply to mostly opaque scenes with smooth view-dependent residuals; mirror-like, transparent, or strongly camera-processed scenes may require geometry or material priors.

\paragraph{Effect of geometric regularization.}
Our analysis studies the image reconstruction objective with bounded angular radiance capacity. Geometry-aware variants such as surface-aligned Gaussian methods, 2D Gaussian formulations, SDF-coupled methods, or normal/depth regularization modify this problem by either restricting the admissible geometry class or adding explicit penalties against off-surface solutions. Such priors can shrink or remove the billboard-favorable regime predicted by our theorem. Thus, our results should not be read as a claim that geometry-aware methods fail in the same way; rather, they explain why such geometric priors are valuable when angular appearance capacity is large.

% References (minimal placeholder)
% ============================================================
\bibliographystyle{splncs04}
\bibliography{main}

% ============================================================
% Supplementary material is included in the same arXiv PDF.
% ============================================================
\clearpage
\appendix
\section*{Supplementary Material}
\addcontentsline{toc}{section}{Supplementary Material}
\renewcommand*{\theHsection}{app.\thesection}
\renewcommand*{\theHsubsection}{app.\thesubsection}

\section{Proof of the Pointwise Projection Proposition}
\label{app:proj}
% ------------------------------------------------------------

Here we prove the pointwise projection proposition in the main paper.

\subsection{Measurable setup and induced hit-point measures}

Let the \emph{ray space} be
\[
\mathcal{R}:=\R^3\times \Sph,
\]
and let $(Y,\hat V)\sim \PP$ be the dataset ray distribution. Fix a candidate surface
$U\subset\R^3$ and assume the first-hit map
\[
x_U:\ \mathcal{R}\to U,\qquad (y,\hat v)\mapsto x_U(y,\hat v)
\]
is measurable on the subset of rays for which the first hit exists and is unique.\footnote{If
non-uniqueness occurs on a $\PP$-null set (e.g., tangential rays), it can be resolved by any measurable
tie-breaking rule without affecting the population loss.}

Define the random hit-point variable
\[
X_U := x_U(Y,\hat V)\in U,
\]
and the joint pushforward measure on $U\times \Sph$,
\[
\nu_U \ :=\ (X_U,\hat V)_\#\PP,
\]
i.e., for any bounded measurable $\varphi:U\times\Sph\to\R$,
\[
\int_{U\times\Sph} \varphi(x,\hat v)\, d\nu_U(x,\hat v)
\ =\
\E_{(Y,\hat V)\sim\PP}\big[\varphi\big(X_U,\hat V\big)\big].
\]
Let $\mu_U$ denote the marginal of $\nu_U$ on $U$:
\[
\mu_U := (\pi_U)_\# \nu_U, \qquad \pi_U(x,\hat v)=x.
\]
By the disintegration theorem (existence of regular conditional probabilities on standard Borel
spaces), there exists a $\mu_U$-a.e.\ defined family of conditional probability measures
$\{\nu_U(\cdot\mid x)\}_{x\in U}$ on $\Sph$ such that for every bounded measurable $\psi$,
\begin{equation}
\int_{U\times\Sph} \psi(x,\hat v)\, d\nu_U(x,\hat v)
=
\int_U\left(\int_{\Sph}\psi(x,\hat v)\, d\nu_U(\hat v\mid x)\right)\, d\mu_U(x).
\label{eq:disintegration}
\end{equation}

\subsection{Replacing raw pixels by the through-seen target via conditional expectation}

Recall the ground-truth pixel random variable along a sampled ray:
\[
Z := I^*(Y,\hat V)\in\R \quad (\text{or }\R^3\text{ channel-wise}).
\]
For fixed geometry $U$, define the \emph{through-seen ground-truth target} as the regular
conditional expectation
\begin{equation}
F_U(x,\hat v)\ :=\ \E\!\left[\,Z\ \middle|\ X_U=x,\ \hat V=\hat v\,\right],
\label{eq:FU_appendix}
\end{equation}
interpreted as any $\nu_U$-a.e.\ version of the conditional expectation
$\E[Z\,|\,\sigma(X_U,\hat V)]$ evaluated at $(x,\hat v)$.

For any admissible appearance function $g$ on $U$ (with $g(x,\cdot)\in\cV_L$ for all $x$),
the population loss is
\[
\cJ(U,g)
=\E\Big[\big|g(X_U,\hat V)-Z\big|^2\Big].
\]
Applying the tower property conditional on $(X_U,\hat V)$ gives
\begin{align}
\cJ(U,g)
&=\E\Big[\,\E\big[\big|g(X_U,\hat V)-Z\big|^2\ \big|\ X_U,\hat V\big]\,\Big]\\
&=\E\Big[\ \big|g(X_U,\hat V)-F_U(X_U,\hat V)\big|^2
\ +\ \Var\!\big(Z\mid X_U,\hat V\big)\ \Big].
\label{eq:tower_mse}
\end{align}
The conditional variance term is independent of $g$. Hence minimizing $\cJ(U,g)$ over admissible
$g$ is equivalent to minimizing
\begin{equation}
\E\Big[\big|g(X_U,\hat V)-F_U(X_U,\hat V)\big|^2\Big]
=
\int_{U\times\Sph} \big|g(x,\hat v)-F_U(x,\hat v)\big|^2\, d\nu_U(x,\hat v).
\label{eq:loss_with_FU}
\end{equation}

\subsection{Pointwise decoupling of the minimization and the weighted projection}

Using the disintegration \eqref{eq:disintegration} in \eqref{eq:loss_with_FU}, we rewrite:
\begin{equation}
\cJ(U,g) \equiv
\int_U
\underbrace{
\left[
\int_{\Sph}\big|g(x,\hat v)-F_U(x,\hat v)\big|^2\, d\nu_U(\hat v\mid x)
\right]}_{=:~\Phi_x(g(x,\cdot))}
\, d\mu_U(x)
\quad +\ \text{const}(U).
\label{eq:loss_pointwise}
\end{equation}
Crucially, for fixed geometry $U$, the function $g(x,\cdot)$ appears \emph{only} inside the term
$\Phi_x(\,\cdot\,)$ indexed by the same point $x$. Therefore, the global minimization decouples:
\begin{equation}
\inf_{g:\ g(x,\cdot)\in\cV_L}\ \cJ(U,g)
\quad\Longleftrightarrow\quad
\text{for $\mu_U$-a.e.\ $x$, solve }\ \min_{h\in\cV_L}\ \Phi_x(h).
\label{eq:decouple}
\end{equation}

For each fixed $x$, define the weighted Hilbert space
$L^2(\Sph,\nu_U(\cdot\mid x))$ with inner product
\[
\langle f,h\rangle_x
:=
\int_{\Sph} f(\hat v)\,h(\hat v)\, d\nu_U(\hat v\mid x),
\qquad
\|f\|_x^2:=\langle f,f\rangle_x.
\]
Since $\cV_L$ is finite-dimensional, it is a closed subspace of this Hilbert space, and the standard
least-squares theorem implies:
\begin{equation}
h_x^*(\cdot)\ :=\ \arg\min_{h\in\cV_L}\ \|h-F_U(x,\cdot)\|_x^2
\ =\ \Pi^{(x)}_{\cV_L}\big(F_U(x,\cdot)\big),
\label{eq:proj_solution}
\end{equation}
where $\Pi^{(x)}_{\cV_L}$ denotes the orthogonal projection w.r.t.\ $\langle\cdot,\cdot\rangle_x$.

Define the global minimizer pointwise by $g^*(x,\cdot):=h_x^*(\cdot)$. Plugging back into
\eqref{eq:loss_pointwise} yields the minimal achievable loss:
\begin{equation}
\inf_{g:\ g(x,\cdot)\in\cV_L}\ \cJ(U,g)
=
\E_{x\sim\mu_U}\Big[\,\dist\big(F_U(x,\cdot),\cV_L\big)^2\,\Big]
\quad +\ \text{const}(U),
\label{eq:min_loss_dist}
\end{equation}
where the distance is taken in the weighted $L^2(\Sph,\nu_U(\cdot\mid x))$ norm.

This proves the pointwise projection proposition in the main paper (absorbing the $g$-independent variance term into the
definition of $B_U(L)$, or equivalently by defining $F_U$ as the conditional expectation target).

\subsection{Explicit coefficient form (optional, useful for intuition)}

Write a basis of $\cV_L$ as $\{\psi_q\}_{q=1}^{Q}$ where $Q=(L+1)^2$ (e.g.\ $\psi_q=Y_{\ell m}$).
For fixed $x$, represent $h(\hat v)=\sum_{q=1}^Q a_q(x)\psi_q(\hat v)$. Then
\eqref{eq:proj_solution} is equivalent to the normal equations:
\begin{equation}
\sum_{q'=1}^Q G_{qq'}(x)\,a_{q'}(x)=b_q(x),
\qquad
G_{qq'}(x):=\langle\psi_q,\psi_{q'}\rangle_x,\quad
b_q(x):=\langle F_U(x,\cdot),\psi_q\rangle_x.
\label{eq:normal_equations}
\end{equation}
If $\nu_U(\cdot\mid x)$ is uniform on $\Sph$, then $G(x)$ is the identity for an orthonormal SH basis,
and $a_q(x)=b_q(x)$ are exactly the SH coefficients up to degree $L$.

% ------------------------------------------------------------
\section{Sobolev Smoothness and Spherical Harmonic Truncation}
\label{app:sobolev}
% ------------------------------------------------------------

This appendix provides the detailed derivation underlying the Sobolev truncation lemma in the main paper and explains
the origin of the factor $1+\ell(\ell+1)$ from the Laplace--Beltrami spectrum on $\Sph$.

\subsection{Spherical harmonics and Laplace--Beltrami eigenvalues}

Let $d\omega$ be the uniform area measure on $\Sph$. The (complex) spherical harmonics
$\{Y_{\ell m}\}_{\ell\ge 0,\ -\ell\le m\le \ell}$ form a complete orthonormal basis of
$L^2(\Sph,d\omega)$:
\[
\langle Y_{\ell m},Y_{\ell' m'}\rangle_{L^2(\Sph)}=\int_{\Sph} Y_{\ell m}(\hat v)\,\overline{Y_{\ell' m'}(\hat v)}\, d\omega(\hat v)
=\delta_{\ell\ell'}\delta_{mm'}.
\]
(Real spherical harmonics satisfy the same orthonormality and eigenvalue relations; all arguments
below apply verbatim.)

A fundamental identity is that $Y_{\ell m}$ are eigenfunctions of the spherical Laplacian
$\Delta_{\Sph}$:
\begin{equation}
\Delta_{\Sph} Y_{\ell m}\ =\ -\ell(\ell+1)\,Y_{\ell m}.
\label{eq:lap_eigs}
\end{equation}

\subsection{Spectral definition of Sobolev norms on the sphere}

For $f\in L^2(\Sph)$, write its SH expansion
\[
f(\hat v)=\sum_{\ell=0}^\infty\sum_{m=-\ell}^{\ell} f_{\ell m} Y_{\ell m}(\hat v),
\qquad
f_{\ell m}:=\langle f,Y_{\ell m}\rangle_{L^2(\Sph)}.
\]
Define the (spectral) Sobolev norm for $s\ge 0$ by
\begin{equation}
\|f\|_{H^s(\Sph)}^2
:=
\sum_{\ell=0}^\infty\sum_{m=-\ell}^{\ell}
\big(1+\ell(\ell+1)\big)^s\,|f_{\ell m}|^2.
\label{eq:sobolev_spectral}
\end{equation}
The factor $1+\ell(\ell+1)$ arises because $(I-\Delta_{\Sph})Y_{\ell m}=(1+\ell(\ell+1))Y_{\ell m}$
by \eqref{eq:lap_eigs}. For integer $s$, this norm is equivalent (up to constants) to the classical
definition involving tangential derivatives on the sphere.

\subsection{Proof of the Sobolev truncation lemma}

Here we prove the Sobolev truncation lemma in the main paper.

Let $P_{\le L}$ denote the $L^2(\Sph)$ orthogonal projection onto
\[
\cV_L=\mathrm{span}\{Y_{\ell m}:0\le \ell\le L,\ -\ell\le m\le \ell\},
\]
i.e.,
\[
(P_{\le L} f)(\hat v)=\sum_{\ell=0}^{L}\sum_{m=-\ell}^{\ell} f_{\ell m}Y_{\ell m}(\hat v).
\]
By orthonormality, the truncation error energy is exactly the discarded coefficient tail:
\begin{equation}
\|f-P_{\le L}f\|_{L^2(\Sph)}^2
=
\sum_{\ell=L+1}^\infty\sum_{m=-\ell}^{\ell}|f_{\ell m}|^2.
\label{eq:tail_energy}
\end{equation}
For every $\ell>L$, we have $(1+\ell(\ell+1))^s\ge (1+L(L+1))^s$, hence
\[
|f_{\ell m}|^2
\le (1+L(L+1))^{-s}\,(1+\ell(\ell+1))^s\,|f_{\ell m}|^2.
\]
Summing over $\ell>L$ and $m$ and using \eqref{eq:tail_energy} yields
\begin{align*}
\|f-P_{\le L}f\|_{L^2(\Sph)}^2
&\le
(1+L(L+1))^{-s}
\sum_{\ell=0}^\infty\sum_{m=-\ell}^{\ell}
(1+\ell(\ell+1))^s|f_{\ell m}|^2\\
&=
(1+L(L+1))^{-s}\,\|f\|_{H^s(\Sph)}^2,
\end{align*}
which is exactly the Sobolev truncation lemma in the main paper.

\subsection{Remark on non-uniform view-direction sampling (why we say ``up to constants'')}

In the paper, projections at each surface point $x$ are taken w.r.t.\ the conditional direction
distribution $\nu_U(\cdot\mid x)$, not necessarily the uniform measure $d\omega$.
If $\nu_U(\cdot\mid x)$ admits a density $w_x(\hat v)$ w.r.t.\ $d\omega$ satisfying
$0<w_{\min}\le w_x(\hat v)\le w_{\max}<\infty$ uniformly on the region of interest, then the
weighted norm $\|f\|_{L^2(\nu_U(\cdot\mid x))}$ is equivalent to $\|f\|_{L^2(d\omega)}$ with
constants $w_{\min},w_{\max}$. Consequently, all truncation bounds transfer to the weighted
inner products with only multiplicative constants depending on these density bounds, justifying
the ``up to view-sampling constants'' language in the true-surface upper-bound proposition in the main paper.

% ------------------------------------------------------------
\section{Parallax Calculus and Frequency Lower Bounds for Misaligned Surfaces}
\label{app:parallax}
% ------------------------------------------------------------

This appendix provides (i) the explicit parallax derivation behind the parallax hit-point relation in the main paper and
the induced-frequency relation in the main paper, (ii) a detailed proof of the great-circle bandlimit lemma in the main paper, and (iii) the
Fourier approximation lower bound used in the Fourier lower-bound lemma in the main paper and
the misaligned-surface lower-bound proposition in the main paper.

\subsection{Deriving the parallax hit-point relation}

Here we derive the parallax hit-point relation in the main paper.

Consider the standard 2D cross-section model (a local linearization of two nearby surface patches):
the camera center is at the origin, and we use coordinates $(x,z)$ in the viewing plane.

Let the \emph{true} surface be the line $z=z_{\mathrm{gt}}$ and the \emph{candidate} (misaligned)
surface be $z=z_{\mathrm{fake}}$. Define the depth gap
\[
\Delta z := z_{\mathrm{gt}}-z_{\mathrm{fake}}.
\]
Parameterize rays by an angle $\alpha$ (relative to the positive $z$-axis). A ray from the origin
with angle $\alpha$ has direction proportional to $(\tan\alpha,1)$, hence can be written as
\[
(x,z)=(t\tan\alpha,t),\qquad t>0.
\]
This ray intersects $z=z_{\mathrm{fake}}$ at $t=z_{\mathrm{fake}}$, with $x$-coordinate
\[
x_{\mathrm{fake}}(\alpha)=z_{\mathrm{fake}}\tan\alpha.
\]
Similarly, it intersects $z=z_{\mathrm{gt}}$ at $t=z_{\mathrm{gt}}$, with coordinate
\[
x_{\mathrm{gt}}(\alpha)=z_{\mathrm{gt}}\tan\alpha.
\]
Now fix a specific point on the fake surface with coordinate $x_0$.
For the ray that hits this fake point, we have $x_{\mathrm{fake}}(\alpha)=x_0$, hence
\[
x_0=z_{\mathrm{fake}}\tan\alpha.
\]
For that same ray, the true intersection point is
\begin{align}
x_{\mathrm{hit}}(\alpha)
&=x_{\mathrm{gt}}(\alpha)
=z_{\mathrm{gt}}\tan\alpha
=z_{\mathrm{fake}}\tan\alpha+(z_{\mathrm{gt}}-z_{\mathrm{fake}})\tan\alpha\\
&=x_0+\Delta z\,\tan\alpha,
\end{align}
which is the parallax hit-point relation in the main paper.

\subsection{From spatial texture to angular frequency}

Assume the diffuse texture on the true surface contains a sinusoidal component along the
tangential coordinate $u$:
\[
T(u)\ \supset\ a\sin(ku),
\qquad a>0,\ k>0,
\]
as in the texture assumption in the main paper. The through-seen target seen at the fake point $x_0$ (ignoring other
terms for the moment) includes
\[
a\sin\big(k\,x_{\mathrm{hit}}(\alpha)\big)
=
a\sin\!\big(k(x_0+\Delta z\tan\alpha)\big).
\]
On a moderate angular range where $|\alpha|$ stays bounded away from $\pi/2$, we may locally
linearize $\tan\alpha\approx \alpha$, yielding
\[
a\sin\!\big(kx_0 + (k\Delta z)\alpha\big),
\]
which is a sinusoid in the angular variable $\alpha$ with angular frequency
\[
\omega_{\mathrm{fake}}\ \approx\ k\Delta z,
\]
as stated in the induced-frequency relation in the main paper. This computation is local, and therefore applies to general
smooth surfaces by restricting to a sufficiently small patch and a sufficiently small viewing arc
where the planar approximation is valid.

\subsection{Proof of the great-circle bandlimit lemma}

Here we prove the great-circle bandlimit lemma in the main paper.

Let $h\in\cV_L$, so
\[
h(\hat v)=\sum_{\ell=0}^L\sum_{m=-\ell}^{\ell} c_{\ell m}Y_{\ell m}(\hat v).
\]
Because $\cV_L$ is rotation-invariant, it suffices to prove the claim for the equator
$\theta=\pi/2$ (any great circle can be rotated to the equator, and rotation preserves degree).

Write $\hat v$ in spherical coordinates $(\theta,\varphi)$. For the standard complex SH basis,
\[
Y_{\ell m}(\theta,\varphi)=N_{\ell m}\,P_\ell^m(\cos\theta)\,e^{im\varphi},
\]
so restricting to the equator ($\theta=\pi/2$, $\cos\theta=0$) yields
\[
Y_{\ell m}\!\left(\tfrac{\pi}{2},\varphi\right)=\underbrace{N_{\ell m}\,P_\ell^m(0)}_{\text{a constant in }\varphi}\,e^{im\varphi}.
\]
Therefore the restriction $\tilde h(\varphi):=h(\pi/2,\varphi)$ has the form
\[
\tilde h(\varphi)
=\sum_{\ell=0}^L\sum_{m=-\ell}^{\ell} \tilde c_{\ell m}\,e^{im\varphi}
=\sum_{m=-L}^{L} d_m e^{im\varphi},
\]
which is a trigonometric polynomial whose highest Fourier frequency is $L$.
The same conclusion holds for a real SH basis, since each real mode at order $m$ is a linear
combination of $\cos(m\varphi)$ and $\sin(m\varphi)$.

\subsection{Proof of the Fourier lower-bound lemma}

Here we prove the Fourier lower-bound lemma in the main paper.

Consider the space of trigonometric polynomials of degree at most $L$:
\[
\mathcal{T}_L := \mathrm{span}\{e^{im\varphi}:\ |m|\le L\}
\subset L^2([0,2\pi]).
\]
Let $f(\varphi)=a\sin(\omega\varphi)$ with $a\neq 0$ and integer $\omega$. Using
\[
\sin(\omega\varphi)=\frac{e^{i\omega\varphi}-e^{-i\omega\varphi}}{2i},
\]
we see $f$ lives entirely in the frequency modes $\pm\omega$. If $\omega>L$, then $f$ is orthogonal
to all basis functions $e^{im\varphi}$ with $|m|\le L$, hence orthogonal to the entire subspace
$\mathcal{T}_L$. Therefore its orthogonal projection onto $\mathcal{T}_L$ is the zero function,
and the best approximation error equals its own energy:
\[
\inf_{p\in\mathcal{T}_L}\ \|f-p\|_{L^2([0,2\pi])}^2
=\|f\|_{L^2([0,2\pi])}^2
=
\int_0^{2\pi} a^2\sin^2(\omega\varphi)\,d\varphi
=
a^2\pi.
\]
This proves the Fourier lower-bound lemma in the main paper.

\subsection{Robustness to additional terms (triangle inequality for distances to a subspace)}

A simple inequality repeatedly used in the fake-geometry argument is:
\begin{equation}
\dist(f+g,V)\ \ge\ \dist(f,V)\ -\ \|g\|,
\label{eq:dist_triangle}
\end{equation}
for any normed space, any subset $V$, and any functions $f,g$.
Indeed, for any $v\in V$,
\[
\|f+g-v\|\ge \|f-v\|-\|g\|,
\]
and taking the infimum over $v\in V$ yields \eqref{eq:dist_triangle}.

In our setting, $f$ is the transported high-frequency component from the diffuse texture,
and $g$ collects all remaining terms (including the residual $R$ and other texture frequencies).
Under the residual-amplitude assumption in the main paper, the $g$ term cannot completely cancel the amplitude of $f$ on the
relevant set of rays, so the constant-scale lower bound survives.

\subsection{From pointwise lower bounds to a global wrong-surface lower bound}

To connect the 1D Fourier lower bound to $B_U(L)$ in the misaligned-surface lower-bound proposition in the main paper, one uses:

\begin{itemize}
\item The pointwise projection characterization from the pointwise projection proposition in the main paper, which reduces
$B_U(L)$ to an average of pointwise squared distances
$\dist(F_U(x,\cdot),\cV_L)^2$ under $\mu_U$.
\item The parallax computation above, which exhibits (on a visible patch with nonzero probability)
a great-circle restriction of $F_U(x,\cdot)$ containing a sinusoid at frequency
$\omega_{\mathrm{fake}}\asymp k\Delta z$ (modulo the mild-angle approximation).
\item The bandlimit lemma the great-circle bandlimit lemma in the main paper, implying any $g(x,\cdot)\in\cV_L$
restricts to a trigonometric polynomial of degree at most $L$ on that viewing circle.
\item The Fourier lower bound the Fourier lower-bound lemma in the main paper, implying an unavoidable constant residual
whenever $L<\omega_{\mathrm{fake}}$, which yields a strictly positive lower bound on
$\dist(F_U(x,\cdot),\cV_L)^2$ for those $x$ and directions.
\end{itemize}

Finally, since the ray distribution $\PP$ assigns positive probability mass to that visible patch
(the texture assumption in the main paper and the misalignment assumption in the main paper), the global expectation defining $B_U(L)$ inherits a
strictly positive constant lower bound $c_0>0$ for all $L<\omega_{\mathrm{fake}}$, completing the
derivation behind the misaligned-surface lower-bound proposition in the main paper.

\section{Numeric Example for the Scale Law}\label{app:numeric_example}

To make the threshold $\omega_{\mathrm{fake}}\approx k\Delta z$ tangible, consider a scene featuring a brick wall with grout lines spaced $p=5$mm ($0.005$m) apart. The spatial angular frequency of this texture component is:
\[
k=\frac{2\pi}{p}\approx \frac{2\pi}{0.005}\approx 1256\ \text{rad/m}.
\]
If the optimizer attempts to place a floater just $5$cm ($0.05$m) in front of this wall, the parallax effect demands an induced angular frequency of:
\[
\omega_{\mathrm{fake}}\approx k\Delta z\approx 1256\times 0.05\approx 62.8.
\]
To fake this specific geometry, the wrong surface must express angular variations on the order of frequency $60$. In stark contrast, the true surface does not need to convert this spatial texture into an angular signal at all; it merely needs to fit the residual $R$, which is typically much smoother.

This calculation is not intended to establish $L=60$ as a universal threshold. Rather, it illustrates the severe scaling dynamics of the $\omega_{\mathrm{fake}}$ requirement. Fine textures combined with even modest geometric misalignments easily push the required angular frequency far beyond typical low SH degrees (e.g., $L=2$ or $3$), creating a robust and highly practical intermediate window where the true geometry is uniquely identifiable.

% ============================================================

% ------------------------------------------------------------
\section{Large-$L$ Achievability: Billboards Can Match the Data}
\label{app:largeL}
% ------------------------------------------------------------

This appendix complements the lower bound in Appendix~\ref{app:parallax} by
showing that once the angular capacity exceeds the parallax-induced frequency
$\omega_{\mathrm{fake}}$, a misaligned opaque billboard can fit the same images
with equally low (indeed, zero) population loss under the first-hit model.

\subsection{An idealized linear-parallax billboard setup}

We work in the same 2D cross-section as Appendix~\ref{app:parallax}.
Let the true surface be the line $S=\{z=z_{\mathrm{gt}}\}$ and the candidate
billboard be the line $U=\{z=z_{\mathrm{fake}}\}$ with depth gap
$\Delta z:=z_{\mathrm{gt}}-z_{\mathrm{fake}}>0$.

Assume the ground-truth appearance is \emph{purely diffuse} with a single spatial
Fourier mode along the tangential coordinate:
\begin{equation}
B(x,\hat v)=T(x)=a\sin(kx),\qquad a>0,\ k>0.
\label{eq:diffuse_single_mode}
\end{equation}
(Thus $A(L)=0$ for all $L\ge 0$ since $T$ is view-independent.)

Fix a point $x_0$ on the billboard $U$. Let the set of training rays that hit
$x_0$ on $U$ arrive from a one-parameter family of viewing directions lying on
a great circle, parameterized by an angle $\varphi\in[0,2\pi]$.
On this family we use the same \emph{linearized parallax} relation as in the main text,
i.e.\ we replace $\tan(\cdot)$ by its first-order approximation:
\begin{equation}
x_{\mathrm{hit}}(\varphi)=x_0+\Delta z\,\varphi.
\label{eq:linear_parallax}
\end{equation}
Therefore the through-seen target at the billboard point $x_0$ is
\begin{equation}
F_U(x_0,\varphi)=T\!\left(x_{\mathrm{hit}}(\varphi)\right)
=a\sin\!\big(kx_0+(k\Delta z)\varphi\big).
\label{eq:FU_billboard_sinusoid}
\end{equation}
Define the induced angular frequency
\begin{equation}
\omega_{\mathrm{fake}}:=k\Delta z.
\label{eq:omega_def_exact}
\end{equation}

\subsection{A surjectivity lemma: degree-$L$ SH can realize any degree-$L$ trigonometric polynomial on a great circle}

\begin{lemma}[Realizing trigonometric polynomials on a great circle]
\label{lem:circle_surject}
Let $C$ be any great circle on $\mathbb S^2$ with arc-length parameter $\varphi\in[0,2\pi]$.
For any trigonometric polynomial
$p(\varphi)=\sum_{m=-L}^{L} d_m e^{im\varphi}$,
there exists a spherical harmonic polynomial $h\in\mathcal V_L$
such that $h$ restricted to $C$ equals $p$:
$h|_C(\varphi)=p(\varphi)$ for all $\varphi$.
\end{lemma}

\begin{proof}
By rotation invariance of $\mathcal V_L$, it suffices to treat the equator
$C=\{(\theta,\varphi):\theta=\pi/2\}$.
For the complex SH basis,
$Y_{\ell m}(\theta,\varphi)=N_{\ell m}P_\ell^{m}(\cos\theta)e^{im\varphi}$.
Restricting to $\theta=\pi/2$ gives
$Y_{\ell m}(\pi/2,\varphi)=N_{\ell m}P_\ell^{m}(0)e^{im\varphi}$.
For each integer $m$ with $|m|\le L$, take $\ell=|m|$.
Then $P_{|m|}^{|m|}(0)=(-1)^{|m|}(2|m|-1)!!\neq 0$, hence
$Y_{|m|,m}(\pi/2,\varphi)=c_m e^{im\varphi}$ with $c_m\neq 0$.
Therefore, setting
$h(\hat v)=\sum_{m=-L}^{L} (d_m/c_m)\,Y_{|m|,m}(\hat v)$
yields $h\in\mathcal V_L$ and $h(\pi/2,\varphi)=p(\varphi)$.
\end{proof}

\subsection{Exact achievability once $L\ge \omega_{\mathrm{fake}}$}

\begin{proposition}[Billboard achieves zero loss for large enough $L$]
\label{prop:billboard_zero_loss}
In the setup above, assume $\omega_{\mathrm{fake}}=k\Delta z$ is an integer.
Then for any angular capacity $L\ge \omega_{\mathrm{fake}}$,
there exists an appearance function $g$ on the billboard $U$ with
$g(x,\cdot)\in\mathcal V_L$ for all $x\in U$ such that the first-hit predictions
match the ground truth almost surely:
\[
I_{U,g}(Y,\hat V)=I^*(Y,\hat V).
\]
Consequently, the optimal population loss is zero:
\[
B_U(L)=0.
\]
Since in the pure diffuse case $A(L)=0$ for all $L$, we have $A(L)=B_U(L)=0$ for all
$L\ge \omega_{\mathrm{fake}}$.
\end{proposition}

\begin{proof}
Fix $x_0\in U$. By \eqref{eq:FU_billboard_sinusoid},
the required angular signal at $x_0$ is the trigonometric polynomial
\[
F_U(x_0,\varphi)=a\sin(kx_0+\omega_{\mathrm{fake}}\varphi)
=a\sin(kx_0)\cos(\omega_{\mathrm{fake}}\varphi)+a\cos(kx_0)\sin(\omega_{\mathrm{fake}}\varphi).
\]
This has degree $\omega_{\mathrm{fake}}$ in $\varphi$. Since $L\ge \omega_{\mathrm{fake}}$,
Lemma~\ref{lem:circle_surject} implies there exists $h_{x_0}\in\mathcal V_L$
whose restriction to the viewing great circle equals $F_U(x_0,\varphi)$.
Define $g(x_0,\hat v):=h_{x_0}(\hat v)$.
Repeating this construction pointwise for each $x\in U$ yields a measurable
$g:U\times\mathbb S^2\to\mathbb R$ with $g(x,\cdot)\in\mathcal V_L$.

By construction, for every training ray that first hits $x_0$ on $U$ with direction
parameter $\varphi$, we have
$g(x_0,\hat v(\varphi))=F_U(x_0,\varphi)=I^*(y,\hat v(\varphi))$,
hence $|I_{U,g}-I^*|^2=0$ on that ray. Therefore the population MSE is zero,
so $B_U(L)=0$.
\end{proof}

\paragraph{Extension to multiple texture frequencies.}
If $T(x)$ is a finite Fourier series $T(x)=\sum_{j=1}^J a_j\sin(k_j x+\phi_j)$,
then the same proof shows $B_U(L)=0$ provided $L\ge \max_j (k_j\Delta z)$ (integer frequencies).

\section{Experimental Details and Protocols}\label{app:exp_details}
In this section, we provide the experimental protocols, evaluation metrics, and categorization criteria omitted from the main text due to space constraints.

\subsection{Training Setup and Optimization}
Across all synthetic and real-world experiments, we use the standard 3D Gaussian Splatting optimization pipeline. We keep the core optimization schedules, including densification intervals, opacity reset frequencies, and learning-rate decays, at their default values. The independent variable in our ablation is the maximum spherical-harmonics degree $L$, which we vary up to $24$. For high-capacity runs, the SH coefficients are optimized together with geometry to test whether the model uses excess angular capacity to fit appearance while leaving the true surface.

\subsection{Geometric Evaluation and Normalization}
We convert learned Gaussians into a point cloud by extracting their 3D centers. We first apply a similarity alignment (Umeyama algorithm) for global scale and rotation, followed by ICP alignment to the ground-truth mesh. We then sample $10^5$ points from the aligned Gaussian centers and $10^5$ points from the ground-truth surface.

To reduce sensitivity to distant floaters, we compute bidirectional point-to-point distances and truncate at the 99th percentile. We report normalized symmetric Chamfer distance,
\[
\text{nCD} = \frac{\text{CD}}{\text{diag}(\text{bbox}(\text{GT}))},
\]
where lower nCD indicates that Gaussian centers lie closer to the true surface.

\subsection{Billboard Labeling Criterion}
For each object, we identify the best appearance and geometry across all runs, denoted $\text{PSNR}_{\text{best}}$ and $\text{nCD}_{\text{best}}$. We classify each run as follows.

\paragraph{Opaque billboard.}
The run has near-optimal rendering quality,
\[
\text{PSNR} \ge \text{PSNR}_{\text{best}} - 2\ \mathrm{dB},
\]
but severely degraded geometry,
\[
\text{nCD} \ge 5 \times \text{nCD}_{\text{best}}.
\]

\paragraph{Billboard onset.}
The run maintains near-optimal rendering quality but has moderate geometric drift:
\[
2 \times \text{nCD}_{\text{best}} < \text{nCD} < 5 \times \text{nCD}_{\text{best}}.
\]

\paragraph{Surface-consistent.}
The run remains geometrically stable, with $\text{nCD} \le 2 \times \text{nCD}_{\text{best}}$.

\end{document}